\documentclass[12pt, draftclsnofoot, onecolumn]{IEEEtran}
\usepackage{amssymb}
\usepackage{amsmath} 
\usepackage[lined,boxed,commentsnumbered, ruled]{algorithm2e}
\usepackage{mathrsfs}
\usepackage{algorithmic}

\usepackage{bm}
\usepackage{tikz}
\usetikzlibrary{arrows}
\usepackage{subfigure}
\usepackage{graphicx,booktabs,multirow}
\usepackage{epstopdf}
\usepackage{subfigure}
\usepackage{multirow}
\usepackage{tabularx} 
\usepackage{booktabs}

\SetKwComment{Comment}{$\triangleright$\ }{}

\definecolor{colorhkust}{RGB}{20,43,140}
\definecolor{colortsinghua}{RGB}{116,52,129}
\definecolor{color1}{RGB}{128,0,0}

\usepackage{amsthm}

\newtheorem{lemma}{Lemma}
\newtheorem{theorem}{Theorem}

\newtheorem{remark}{Remark}

\newcommand{\Ex}{\mathbb{E}}

\newcommand{\comment}[1]{}
\newcommand{\Beta}{\text{Beta}}

\newcommand{\remove}[1]{\textcolor{red}{}}

\newcommand{\add}[1]{#1}

\begin{document}
	\title{Faster Activity and Data Detection in Massive Random Access: A Multi-armed Bandit Approach}
\author{Jialin Dong, \textit{Student Member}, \textit{IEEE},  Jun Zhang, \textit{Senior Member}, \textit{IEEE},  Yuanming~Shi, \textit{Member}, \textit{IEEE},\\ and Jessie Hui Wang
	
	\thanks{J. Dong is with the School of Information Science and Technology, ShanghaiTech University, Shanghai 201210, China, and the Department of Electronic and Information Engineering, The Hong Kong Polytechnic University, Hong Kong (e-mail: dongjl@shanghaitech.edu.cn). 
		
		J. Zhang is with the the Department of Electronic and Information Engineering, The Hong Kong Polytechnic University, Hong Kong (e-mail: jun-eie.zhang@polyu.edu.hk).
		
		 Y. Shi is with the School of Information Science and Technology, ShanghaiTech University, Shanghai 201210, China. (e-mail: shiym@shanghaitech.edu.cn).
		 
		  J. H. Wang is with the Institute for Network Sciences and Cyberspace, Tsinghua University, Beijing 100084, China, and also with the Beijing National Research Center for Information Science and Technology, Beijing 100084, China (e-mail: jessiewang@tsinghua.edu.cn). (The corresponding author is J. H. Wang.)}
	
}
	
	\maketitle
\begin{abstract}
This paper investigates the grant-free random access with massive IoT devices. By embedding the data symbols in the signature sequences, joint device activity detection and data decoding can be achieved, which, however, significantly increases the computational complexity.
Coordinate descent algorithms that enjoy a low per-iteration complexity have been employed to solve the detection problem, but previous works typically employ a random coordinate selection policy which leads to slow convergence. In this paper, we develop multi-armed bandit approaches for more efficient detection via coordinate descent, which make a delicate trade-off between \emph{exploration} and \emph{exploitation} in coordinate selection. Specifically, we first propose a bandit based strategy, i.e., Bernoulli sampling, to speed up the convergence rate of coordinate descent, by learning which coordinates will result in more aggressive descent of the objective function. To further improve the convergence rate, an inner multi-armed bandit problem is established to learn the exploration policy of Bernoulli sampling. Both convergence rate analysis and simulation results are provided to show that the proposed bandit based algorithms enjoy faster convergence rates with a lower time complexity compared with the state-of-the-art algorithm. Furthermore, our proposed algorithms are applicable to different scenarios, e.g., massive random access with low-precision analog-to-digital
converters (ADCs).

\end{abstract}

\begin{IEEEkeywords}
Massive connectivity, Internet of Things, coordinate descent, multi-armed bandit, Thompson sampling.
\end{IEEEkeywords}
\section{Introduction}
The advancements in wireless technologies have enabled connecting sensors, mobile devices, and machines for various mobile applications, leading to an era of Internet-of-Things (IoT) \cite{zanella2014internet}. IoT connectivity
involves connecting a massive number of
devices, which form the foundation for many applications, e.g., smart home, smart city,
healthcare, transportation system, etc. Thus it has been regarded as an indispensable demand for future wireless networks \cite{8808168}. With a large number of devices to connect with the base station (BS), in the order $10^4$ to $10^6$, massive connectivity brings formidable technical challenges, and has attracted lots of attentions from both the academia and industry \cite{8454392,hasan2013random}. 

The sporadic traffic is one unique feature in massive IoT connectivity, which means that only a restricted portion of devices are active at any given time instant \cite{liu2018sparse}.
This is because IoT devices are often
designed to sleep most of the time to save energy,
and are activated only when triggered by external events. Therefore, the
BS needs to manage the massive random access via detecting the active users before data transmission. 
The grant-based random access scheme has been widely applied to allow multiple users to access the network over limited radio resources, e.g., in 4G LTE networks \cite{hasan2013random}-\cite{arunabha2010fundamentals}. Under this scheme, each active device is randomly assigned a pilot sequence from a pre-defined set of preamble sequences to notify the BS of the device's activity state. A connection between an active device and the BS will be established if the pilot sequence of this device is not engaged by other devices. Besides the overhead caused by the pilot sequence, a major drawback of the grant-based random access scheme is the collision issue due to a massive number of devices \cite{liu2018sparse}. 

To avoid the excessive access latency due to the collision, a grant-free random access scheme has been proposed \cite{liu2018sparse}. Under this scheme, the active devices do not need to wait for any grant to access the network, and can directly transmit the payload data following the metadata to the BS. Following activity detection and channel estimation based on the pilot sequences, payload data of the active devices can be decoded.
The key idea of activity detection and data decoding under the sporadic pattern
is to connect with sparse signal processing and leverage the compressed sensing techniques \cite{chen2018sparse}. Compared with the grant-based access scheme \cite{liu2018sparse}, the grant-free random access paradigm enjoys a much lower access latency. In the scenario where the payload data only contains a few bits, e.g., sending an alarm signal, the efficiency can be further improved by embedding the data symbols in the signature sequences \cite{senel2018grant,8761672}. Nevertheless, with massive devices and massive BS antennas, the resulting high-dimensional detection problem brings formidable computational challenges, which motivates our investigation.
\subsection{Related Works}
We consider the grant-free massive random
access scheme in a network consisting of one multi-antenna BS and a massive number of devices with small data payloads, where each message is assigned a unique signature sequence.
By exploiting the sparsity structure in both the device activity state and data transmission, joint device activity detection and data decoding can be achieved by leveraging compressed sensing techniques \cite{liu2018massive1,chen2018sparse}. Recently, a covariance-based method has been proposed to improve the performance of device activity detection \cite{haghighatshoar2018improved}, where the detection problem is solved by a coordinate descent algorithm with random sampling, i.e., it randomly selects coordinate-wise iterate to update. This covariance-based method has also been applied for joint detection and data decoding \cite{8761672}. Furthermore, the phase transition analysis for covariance-based massive random access with massive MIMO has been provided in \cite{chen2019phase}.

Although coordinate descent is an effective algorithm to solve the maximum
likelihood estimation problem for joint activity detection and  data decoding \cite{8761672}, existing works adopted a random coordinate selection strategy, which yields a slow convergence rate.
Besides, a rigorous convergence rate analysis for this strategy has not yet been obtained. In this paper, our principle goal is to develop coordinate descent algorithms with more effective coordinate selection strategies for \emph{faster} activity and data detection in massive random access, supported by rigorous convergence rate analysis.


Coordinate descent algorithms \cite{wright2015coordinate} with various coordinate selection strategies have been widely applied to solve optimization  problems for which computing the gradient of the objective function is computationally prohibitive. It enjoys a low per-iteration complexity, as one or a few coordinates are updated in each iteration. In most previous works, e.g., \cite{shalev2013accelerated,shalev2013stochastic}, each coordinate is selected uniformly at random at each time step. Recent studies have proposed more advanced coordinate selection strategies via exploiting the structure of the data and sampling the coordinates from an appropriate non-uniform distribution, e.g., \cite{nutini2015coordinate}-\cite{salehi2018coordinate}, which outperform the random sampling strategy in the convergence rate.

Specifically, a convex optimization problem that minimizes a strongly convex objective function was considered in \cite{nutini2015coordinate}. It proposed a GaussSouthwell-Lipschitz rule that gives a faster
convergence rate than choosing random coordinates.
Subsequently,
Perekrestenko \textit{et al}. \cite{perekrestenko2017faster} improved convergence rates of the coordinate descent in an adaptive scheme on general convex objectives. Additionally, Zhao and Zhang \cite{zhao2015stochastic} developed an importance sampling rule where the sample distribution depends on the Lipschitz constants of the loss functions.
The adaptive sampling strategies in \cite{perekrestenko2017faster,zhao2015stochastic} require the full information of all the coordinates, which yields high computation complexity at each step. To address this issue, a recent study \cite{salehi2018coordinate}  exploited a bandit algorithm to learn a good approximation of the reward function, which characterizes how much the cost function decreases when the corresponding coordinate is updated.
The coordinate descent algorithms proposed in all the works mentioned above are to solve \emph{convex} optimization problems.
 Different from these works, the covariance-based estimation problem is \emph{non-convex}. Hence, efficient algorithms with new reward functions and corresponding theoretical analysis are required, which bring unique challenges.

\subsection{Contributions}
In this paper, we propose coordinate descent algorithms with effective coordinate sampling strategies for faster activity and data detection in massive random access. Specifically, we develop a novel algorithm, i.e., \emph{coordinate descent with Bernoulli sampling}. Inspired by \cite{salehi2018coordinate}, we cast the coordinate selection procedure as a multi-armed bandit (MAB) problem where a reward is received when selecting an arm (i.e., a coordinate), and we aim to maximize the cumulative rewards over iterations. At each iteration, with probability $\varepsilon$ the coordinate with the largest reward is selected, and otherwise the coordinate is chosen uniformly at random. We provide the convergence rate analysis on the coordinate descent with both Bernoulli sampling and random sampling in Theorem \ref{thm:main_bandit}, which theoretically validates the advantages of the proposed algorithm. While the algorithm and analysis in \cite{salehi2018coordinate} only considered convex objective functions, we extend them to the non-convex case. 

The value of $\varepsilon$ plays a vital role in the convergence rate and the computational cost. As demonstrated in Theorem \ref{thm:main_bandit}, the larger the value of $\varepsilon$ is, the higher profitability of selecting the coordinate endowed with the largest reward is. On the other hand, a larger value of $\varepsilon$ leads to a higher computational cost, since the rule of selecting the coordinate with the largest reward requires computing the rewards for all the coordinates. This motivates us to develop a more advanced algorithm called \emph{coordinate descent with Thompson sampling}, which adaptively adjusts the value of $\varepsilon$. In this algorithm, an inner MAB problem is established to learn the optimal value of $\varepsilon$, which is solved by a Thompson sampling algorithm. Theoretical analysis is provided to demonstrate that the logarithmic expected regret for the inner 
MAB problem can is achieved. 
Different from the analysis of Thompson sampling in previous works where the parameters of the beta distribution are required to be integers, i.e., \cite{agrawal2012analysis,scott2010modern}, our analysis applies to the beta distribution of which the parameters are in the more general and natural forms.

Simulation results show that the proposed algorithms enjoy faster convergence rates with lower time complexity than the state-of-the-art algorithm. It is also demonstrated that coordinate descent with Thompson sampling enables to further improve the convergence rate compared to coordinate descent with Bernoulli sampling.
Furthermore, we show that the proposed algorithm can be applied to faster activity and data detection in more general scenarios, i.e., with low precision (e.g., 1 -- 4 bits) analog-to-digital
converters (ADCs).

%
%
\section{System model and problem formulation}\label{sys}
In this section, we introduce the system model for massive random access, a.k.a., massive connectivity. A covariance-based formulation is then presented for joint device activity detection and data decoding, which is solved by a coordinated descent algorithm with random sampling.
\subsection{System Model}
Consider an IoT network consisting of one BS equipped with $M$ antennas and  $N$ single-antenna IoT devices. The channel state vector from device $i$ to the BS is denoted by\begin{align}\label{eq:hg}
g_i\bm{h}_i\in\mathbb{C}^{ M}, \quad i=1, \ldots, N,
\end{align} where $g_i$ is the pathloss component depending
on the device location, and $\bm{h}_i\in\mathbb{C}^{M}$ is the
Rayleigh fading component over multiple antennas that obeys
i.i.d. standard complex Gaussian distribution, i.e., $\bm{h}_i\sim\mathcal{CN}(\bm 0,\bm I)$. Due to the sporadic communications, only a few devices are active out of all devices at a given time instant \cite{wunder2015sparse}. For each active device, $J$ bits of data are transmitted, where $J$ is typically a small number. This is the case for many applications, e.g., sending an alarm signal requires only 1 bit. Our goal is to achieve the joint device activity detection and data detection.

Assume the channel coherence block endows with length $T_c$. The length of the signature sequences $L$ ($L<T_c$) is generally smaller than the number of devices, i.e., $L \ll N$, due to the massive number of devices and a limited channel coherence block \cite{8761672,wunder2015sparse}. We first define a unique signature sequence set for $N$ devices. For each device, we assign each $J$-bit message with a unique sequence. With $R:=2^{J}$, this sequence set is known at the BS:
\begin{align}
\bm{Q} = [\bm{Q}_1~ \cdots~\bm{Q}_N ]\in\mathbb{C}^{L\times NR},\label{eq:Q}
\end{align} 
where $\bm{Q}_i = [\bm{q}_i^1, \cdots, \bm{q}_i^R]\in\mathbb{C}^ {L\times R}$ with $\bm{q}_i^r = [{q}_{i}^r(1),\cdots,$ $ {q}_{i}^r(L)]^{\top}$ $\in\mathbb{C}^{L}$ for $i = 1,\cdots, N, r = 1,\cdots, R$. We assume that all the signature sequences are generated from i.i.d. standard complex Gaussian distribution, and are known to the BS. If the  $i$-th device is active and aims to send a certain data of $J$ bits, the $i$-th device will transmit the corresponding sequence from $\bm{Q}_i$. Specifically, the indicator $a_i^r$ that implies whether the $r$-th sequence of $i$-th device is transmitted is defined as follows: $a_i^r=1$ if the $i$-th device transmits the $r$-th sequence; otherwise, $a_i^r=0$. By detecting which sequences are transmitted
based on the received signal, i.e., estimating $\{a_i^r\}$, the BS
achieves joint activity detection and data decoding. In this way, the information bits are embedded in the transmitted sequence, and no extra payload data need to be transmitted, which is very efficient for transmitting a small number of bits \cite{senel2018grant}. Since at most one sequence is transmitted by each device, it holds that $\sum\nolimits_{r = 1}^R {a_i^r \in \left\{ {0,1} \right\}}$, where $\sum\nolimits_{r = 1}^R {a_i^r = 0}$ indicates that device $i$ is inactive; otherwise, it is active. 
The received signal $\bm{y}(\ell)\in\mathbb{C}^{ M}$ at the BS is represented as
\begin{equation}\label{eq:signal}
\bm{y}(\ell)=\sum_{i=1}^{N}\sum_{r =1}^R\bm{h}_i a_i^r q_i^r(\ell)+\bm n(\ell), 
\end{equation}
where $\bm{n}(\ell)\in\mathbb{C}^{ M}$ is the additive noise such that $\bm{n}(\ell)\sim\mathcal{CN}(\bm 0,\sigma_n^2\bm I)$ for all $\ell=1,\dots, L$. 

Compact the received signal over $M$ antennas as $\bm{Y}=[\bm{y}(1),\dots, \bm{y}(L)]^{\top}\in \mathbb{C}^{L\times M}$, and the additive noise signal over $M$ antennas as
\begin{align}
\bm{N}=[\bm{n}(1),\dots,
	\bm{n}(L)]\in\mathbb{C}^{L\times M}.\label{eq:N}
\end{align}
  The channel matrix is concatenated as \begin{align}
\bm{H}=[\bm{H}_1,\dots, \bm{H}_N]^{\top}\in\mathbb{C}^{NR\times M}\label{eq:H}
\end{align} with $\bm{H}_i = [\bm{h}_i,\cdots,\bm{h}_i]^{\top}\in\mathbb{C}^{R\times M}$ consisting of repeated rows for $n = 1,\cdots, N$. Recall the signature sequences defined in (\ref{eq:Q}), and then the model \eqref{eq:signal} can be reformulated as \cite{8761672}:
\begin{equation}\label{eq:sys_model}
\bm{Y}=\bm{Q}\bm{\Gamma}^{\frac{1}{2}}\bm{H}+\bm{N},
\end{equation}
where the diagonal block matrix is ${{\bm{\Gamma }}^{\frac{1}{2}}} \triangleq {\text{diag}}\left( {{{\bm{D}}_1}, \ldots ,{{\bm{D}}_N}} \right) \in {{\mathbb{C}}^{NR \times NR}}$ with $\bm{D}_i={\text{diag}}(a_i^1g_i,\dots,$ $ a_i^Rg_i)\in\mathbb{C}^{R
	\times R}$ being the diagonal activity matrix of the $i$-th device. Let $\bm{\gamma}=[\bm{\gamma}_1^{\top},\cdots,\bm{\gamma}_{N}^{\top}]^{\top}\in\mathbb{C}^{NR}$ denote the diagonal entries of $\bm{\Gamma}$, where $\bm{\gamma}_i = [(a_i^1g_i)^2,$ $\dots, (a_i^Rg_i)^2]^{\top}\in\mathbb{C}^{R}$ for $i = 1,\cdots,N$.
Our goal is to detect the values of indicators (i.e., $ \{a_i^r\}$) from the received matrix $\bm{Y}$ with the knowledge of the pre-defined sequence matrix $\bm{Q}$.

\subsection{Problem Analysis}
To achieve this goal, recent works have developed a compressed sensing based approach \cite{liu2018massive1,Yuanming_IoT2018device,liu2018massive2} which recovers $\bm{\Gamma}^{\frac{1}{2}}\bm{H}$ from $\bm{Y}$ via exploiting the group sparsity structure of $\bm{\Gamma}^{\frac{1}{2}}\bm{H}$. The indicator $a_i^r$ can then be 
determined from the rows of $\bm{\Gamma}^{\frac{1}{2}}\bm{H}$. However, such an approach
usually suffers an algorithmic complexity that is dominated by $M$ in massive IoT networks, i.e., the high dimension of $\bm{\Gamma}^{\frac{1}{2}}\bm{H}$. Furthermore, with messages embedded in the signature sequences, there is no need to estimate the channel state information \cite{8761672}, and thus recent papers \cite{8761672,haghighatshoar2018improved} have focused on directly detecting activity via estimating $\bm{\Gamma}$ instead.

Specifically, the estimation of $\bm{\Gamma}$ can be formulated as a maximum likelihood
estimation problem. Given $\bm{\gamma}$, each column of $\bm{Y}$, denoted as $\bm{y}_m\in\mathbb{C}^{L}$ for $1\leq m\leq M$, can be termed as an independent sample
from a multivariate complex Gaussian distribution such that \cite{8761672}:
\begin{align}\label{eq:distri}
{{\bm{y}}_m}\sim C{\mathcal{N}}\left( {{\bm{0}},\bm{\Sigma}} \right),
\end{align}
where $
\bm{\Sigma} = {\bm{Q}}{{\bm{\Gamma }}}{{\bm{Q}}^{\mathsf{H}}} + \sigma_n^2{\bm{I}_L}
$
with the identity matrix $\bm{I}_L\in\mathbb{R}^{L\times L}$. Based on (\ref{eq:distri}), the likelihood of $\bm{Y}$ given $\bm{\gamma}$ is represented as \cite{8761672}:
$
P({\bm{Y}}|\bm \gamma ) = \prod\limits_{m = 1}^M {\frac{1}{\mathrm{det}({\pi {\bm{\Sigma }}})}}$ $ \exp ( { - {\bm{y}}_m^{\mathsf{H}}{{\bm{\Sigma }}^{ - 1}}{{\bm{y}}_m}} )  = {{{{( \mathrm{det}({\pi {\bm{\Sigma }}}))}^{-M}}}}\exp ( { - {\mathrm{Tr}}( {{{\bm{\Sigma }}^{ - 1}}{\bm{Y}}{{\bm{Y}}^{\mathsf{H}}}} )} ),
$
where $\mathrm{det}(\cdot )$ and $\mathrm{Tr}(\cdot)$ are operators that return the determinant and the trace of a
matrix, respectively. Based on (\ref{eq:distri}), the maximum likelihood estimation problem can be formulated as minimizing $-\log P({\bm{Y}}|\gamma)$:
\begin{align}\label{eq:mle}\mathop{{\text{minimize}}}\limits_{\bm{\gamma }\in\mathbb{R}^{NR}} &\quad \log |{\bm{\Sigma }}| + \frac{1}{M}{\mathrm{Tr}}\left( {{{\bm{\Sigma }}^{ - 1}}{\bm{Y}}{{\bm{Y}}^{\mathsf{H}}}} \right)\notag \\  {\text{subject}}\;{\text{to}}&\quad {{\bm \gamma }} \geq 0,\notag\\ &\quad {{|| {{{{\bm \gamma }}_i}} ||}_0} \leq 1,\quad i = 1,2, \ldots ,N,\end{align}
where ${{\bm \gamma }} \geq 0$ means that each element of $\bm{\gamma}$ is greater or equal to $0$, and ${||\cdot||}_0$ denotes the $\ell_0$ norm.
This covariance-based approach was first proposed in \cite{haghighatshoar2018improved} for activity detection, and then extended to joint activity and data detection in \cite{8761672}. 
Based on the estimated $\hat{\bm{\gamma}}$ and a pre-defined threshold $s_{t h}$, the indicator can be determined by
\begin{equation}\label{eq:thresh}
a_{i}^{r}=\left\{\begin{array}{ll}{1,} & {\text { if } \quad \hat{\gamma}_{i}^{r} \geq s_{t h} \text { and } \hat{\gamma}_{i}^{r}=\max _{j=1}^{R}\{\hat{\gamma}_{i}^{j}\}}, \\ {0,} & {\text { else. }}\end{array}\right.
\end{equation}
From $a_i^r$ that indicates whether the $r$-th sequence is transmitted by the $i$-th device, the activity state of the $i$-th device and the transmitted data can be determined, i.e., achieving joint activity detection and data decoding.

For the ease of algorithm design, an alternative way to solve problem (\ref{eq:mle}) was developed in \cite{8761672}.  By eluding the absolute value constraints, it yields
\begin{align}
\label{eq:mle_r}
\mathop{{\text{minimize}}}\limits_{{{\bm\gamma }} \geq 0} &\quad F(\bm\gamma):=\log |{\bm{\Sigma }}| + \frac{1}{M}{\mathrm{Tr}}\left( {{{\bm{\Sigma }}^{ - 1}}{\bm{Y}}{{\bm{Y}}^{\mathsf{H}}}} \right).
\end{align}
The first term in (\ref{eq:mle_r}) is a concave function that makes the objective nonconvex, thereby bringing a unique challenge.
The paper \cite{8761672} showed that the estimator $\hat{\bm{\gamma}}$
of problem (\ref{eq:mle_r}) by coordinate descent is approximately
sparse, thus constraints $\quad {{|| {{{{\bm \gamma }}_i}} ||}_0} \leq 1,\forall i$ can be  approximately satisfied. Specifically, it demonstrated that as the sample size, i.e., $L$, increases, the estimator $\hat{\gamma}$ of problem (\ref{eq:mle_r}) concentrates around the ground truth $\bm{\gamma}^{\natural}$ and becomes an approximate sparse vector for large $M$, which
implies that constraints $\quad {{|| {{{{\bm \gamma }}_i}} ||}_0} \leq 1,\forall i$ are satisfied approximately when $M$ is large.
Motivated by its low per-iteration complexity, the papers \cite{8761672,haghighatshoar2018improved} developed a coordinate descent algorithm to solve the relaxed problem (\ref{eq:mle_r}), which updates the coordinate of $\bm{\gamma}$
randomly until convergence (illustrated in Algorithm \ref{tab:ML_coord}). However, such a simple coordinate update rule
\begin{algorithm}[htbp]
	\caption{CD-Random} 
	\label{tab:ML_coord} 
	{
		\begin{algorithmic}[1]
			\STATE {\bf{  Input}:} The sample covariance matrix $\widehat{\bm{\Sigma}}_{\bm{ y}}=\frac{1}{M} \bm{Y} \bm{Y}^{\mathsf{H}}$ of the $L \times M$ matrix $\bm{Y}$.
			\STATE {\bf{Initialize}:} $\bm{\Sigma}=\sigma_n^2 \bm{ I}_{L}$, $\bm{\gamma}={\bm{0}}$.
			
			\FORALL{ $t=1,2, \dots$}
			\STATE {Select an index $k \in [NR]$ corresponding to the $k$-th component of  $\bm\gamma$  randomly.}
%
			\STATE Let $\bm{a}_k$ denote the $k$-th column of $\bm{Q}\in\mathbb{C}^{L\times NR}$, and set $\delta=\max \left\{\frac{\bm{a}_k^{\mathsf{H}} \mathbf{\Sigma}^{-1} \widehat{\mathbf{\Sigma}}_{\bm{y}} \mathbf{\Sigma}^{-1} \bm{a}_k-\bm{a}_k^{\mathsf{H}} \mathbf{\Sigma}^{-1} \bm{a}_k}{\left(\bm{a}_k^{\mathsf{H}} \mathbf{\Sigma}^{-1} \bm{a}_k\right)^{2}},-\gamma_{k}\right\}$\label{5}
			
			\STATE Update $\gamma_{k} \leftarrow \gamma_{k}+ \delta$.
			\STATE Update $\bm{\Sigma} \leftarrow \bm{\Sigma}+ \delta (\bm{a}_{k} \bm a_{k}^{\mathsf{H}})$.\label{7} 
		
			\ENDFOR
			\STATE {\bf { Output}:} $\bm\gamma=[\gamma_1, \dots, \gamma_{ NR}]^{\top}$.
	\end{algorithmic}}
\end{algorithm}
yields a less aggressive convergence rate, and lacks rigorous convergence rate analysis with theoretical guarantees. In this paper, we aim to design a novel sampling strategy for coordinate descent to improve its convergence rate.

There have been lots of efforts in pushing the efficiency of coordinate descent algorithms by developing more sophisticated coordinate update rules. Concerning supervised learning problems, previous works \cite{perekrestenko2017faster,shalev2013stochastic} have demonstrated that the coordinate descent algorithm can yield better convergence guarantees when exploiting the structure of the data and sampling the coordinates from an appropriate non-uniform distribution. Furthermore, the paper \cite{salehi2018coordinate}
proposed a multi-armed
bandit based coordinate selection method that can be
applied to minimize convex objective functions, e.g., Lasso, logistic and ridge regression. Inspired by \cite{salehi2018coordinate}, we shall apply the idea of Bernoulli sampling to solve the estimation problem (\ref{eq:mle_r}) with a \emph{non-convex} objective function for joint activity and data detection. In the remainder of the paper, we first present a basic coordinate descent algorithm with Bernoulli sampling in Section \ref{sec:3}, followed by proposing a more efficient algorithm with Thompson sampling in Section \ref{sec:4}, both with rigorous analysis. Simulation results are provided in Section \ref{sec:5}.

\section{Coordinate Descent with Bernoulli Sampling}\label{sec:3}
In this section, a basic algorithm, coordinate descent with Bernoulli sampling, is developed. 
 We begin with introducing a reward function for each coordinate, which quantifies the decrease of the objective function $F(\bm\gamma)$ in (\ref{eq:mle_r}) by updating the corresponding coordinate. Based on the reward function, a coordinate descent algorithm with Bernoulli sampling (CD-Bernoulli) is proposed for joint device activity and data detection. The convergence rate of the proposed algorithm will be provided, and compared with that of coordinate descent with random sampling \cite{haghighatshoar2018improved}.

\subsection{Reward Function}
The coordinate selection strategy depends on the update rule for the decision variable
$\gamma_k$ for $k\in[NR]$. The update rule with respect to the $k$-th coordinate is denoted as $\mathcal{H}_k$, which is illustrated by Line 5-7 in Algorithm \ref{tab:ML_coord}. The following lemma quantifies the decrease of updating a coordinate
$k\in[NR]$ according to the update rule $\mathcal{ H}_k$, which is the reward function in our proposed algorithm and denoted as $r_k$.
\begin{lemma}\label{lemma:1}
	Considering problem (\ref{eq:mle_r}), and choosing the
	coordinate $k\in[NR]$ and updating $\gamma_k^t$ with the update rule $\mathcal{H}_k$, we have the following bound:
$
	F\left(\boldsymbol{\gamma}^{t+1}\right) \leq F\left(\bm{\gamma}^{t}\right)-r_{k}^{t},
$
	where 
	\begin{align}\label{eq:reward}
	r_{k}^{t} = \frac{\bm{a}_{k}^{\mathsf{H}} \boldsymbol{\Sigma}^{-1} \widehat{\boldsymbol{\Sigma}}_{\bm{y}} \boldsymbol{\Sigma}^{-1} \bm{a}_{k}}{1+\delta \bm{a}_{k}^{\mathsf{H}} \boldsymbol{\Sigma}^{-1} \bm{a}_{k}} \delta-\log\left(1+\delta \bm{a}_{k}^{\mathsf{H}} \boldsymbol{\Sigma}^{-1} \bm{a}_{k}\right).
	\end{align}
\end{lemma}
\begin{proof}
Please refer to Appendix \ref{sec:reward} for details.
\end{proof}
A greedy algorithm based on Lemma \ref{lemma:1} is to simply select at time $t$ the coordinate $k$ with the largest $r_k^t$ at time $t$. However, the cost of computing reward functions for all the $k\in [NR]$ is prohibitively high, especially with a large number of devices.
To address this issue, the paper \cite{salehi2018coordinate} adapted a principled approach using a bandit framework for learning the best $r_k^t$'s, instead of exactly computing all of them. Inspired by this idea, at each step $t$, we select a single coordinate $k$ and update it according to the rule $\mathcal{H}_k$. The reward function $r^t_k$ is computed and used as a feedback to adapt the coordinate selection strategy with Bernoulli sampling.
Thus, only partial information is available for coordinate selection, which reduces the computational complexity of each iteration. Details of the algorithm are provided in the following subsection.
\subsection{Algorithm and Analysis}
Consider a multi-armed bandit (MAB) problem where there are $NR$ arms (coordinates in our setting) from which a bandit algorithm can select for a reward, i.e., $r_k^t$ as in (\ref{eq:reward}) at time $t$. 
The MAB aims to maximize the cumulative reward received over $T$ rounds, i.e., $\sum_{t=1}^T r_{k_t}^t$, where $k_t$ is the arm (coordinate) chosen at time $t$. 
After the $t$-th round, the MAB only receives the reward of the selected arm (coordinate) $k_t$ which is used to adjust its arm (coordinate) selection strategy for the next round. For more background on the MAB problem, please refer to \cite{bubeck2012regret}.

Based on the MAB problem introduced above, the CD-Bernoulli algorithm is illustrated in Algorithm \ref{alg:Bandit}. \begin{algorithm}[tbp]
	\caption{CD-Bernoulli}\label{alg:Bandit}
	\begin{algorithmic}[1] 
		\STATE \textbf{Input: } $\varepsilon$ and $B$
		\STATE \textbf{Initialize: } $\bm{\Sigma}=\sigma_n^2 \bm{ I}_{L}$, $\bm{\gamma}={\bm{0}}$,
		  set $\bar{r}^0_k = r^0_k$ for all $k \in [NR]$.
		\FOR{$t=1$ {\bfseries to} $T$}
		\IF {$t \mod B == 0$}
		\STATE  set $\bar{r}_k^t = r^t_k$ for all $k \in [NR]$
		\ENDIF
		\STATE Generate $K \sim \mathrm{Bernoulli}(\varepsilon)$ 
		\IF {$K==1$}
	\STATE 	Select $k_t = \arg\max_{k\in [NR]} \bar{r}_{k}^{t}$
	
	\ELSE
	\STATE  Select $k_t \in [NR]$ uniformly at random
	\ENDIF
		\STATE	Update $\gamma_{k_t}^{t}$ according to the rule $\mathcal{H}_{k_t}$
		\STATE	Set $\bar{r}_{k_t}^{t+1} = r^{t+1}_{k_t}$ and $\bar{r}_{k}^{t+1} = \bar{r}_{k}^{t}$ for all $k\neq k_t$
		\ENDFOR
	\end{algorithmic}
\end{algorithm}
To address the computational complexity issue of the greedy algorithm that requires to compute the reward function $r^t_k$ for all $k\in [NR]$ at each round $t$, Algorithm \ref{alg:Bandit} only computes the reward function $r^t_k$ of all the coordinates $k\in [NR]$ every $B$ rounds (please refer to Line 4-6 in Algorithm \ref{alg:Bandit}). In the remaining rounds, $\bar{r}_k$ is estimated based on the most recently observed reward in the MAB. The coordinate selection policy is presented as follows: with probability $(1-\varepsilon)$ a coordinate $k_t \in [NR]$ is determined uniformly at random, while with probability $\varepsilon$ the coordinate endowed with the largest $\bar{r}_k^t$ is chosen. It mimics the $\epsilon$-greedy approach for conventional MAB problems \cite{bubeck2012regret}. This is to achieve a tradeoff between \emph{exploration} and \emph{exploitation}. That is, whether choosing the coordinate with currently the largest reward or exploring other coordinates. Then the $k_t$-th coordinate of $\bm{\gamma}$ is updated according to the update rule $\mathcal{H}_{k_t}$. The $k_t$-th entry of the estimated reward function is updated as $\bar{r}_{k_t}^{t+1} = r^{t+1}_{k_t}$ with the rest unchanged.

The following result shows the convergence rate of coordinate descent for joint activity and data detection with two different coordinate selection strategies, i.e., random sampling and Bernoulli sampling. The estimation error is defined as \begin{align}\label{epsilon}
\epsilon(\bm{\gamma}) = F(\bm{\gamma}) - F({\bm{\gamma}}^{\star})
\end{align} with ${\bm{\gamma}}^{\star}:= \text{argmin}_{\bm{\gamma} \in \mathbb{R}^{NR}}F(\bm{\gamma})$. In contrast to the previous work \cite{salehi2018coordinate}, which concerns the objective function consisting of a smooth convex function and a regularized convex function, this paper considers $F(\bm{\gamma})$ in (\ref{eq:mle_r}) that consists of a concave function and a convex function. Denote the best arm (coordinate) as $j_\star^t = \arg\max_{k \in [NR]} \bar{r}_k^t$ with the estimated reward $\bar{r}_k^t$ in Algorithm \ref{alg:Bandit}, we have the following convergence result.
\begin{theorem} \label{thm:main_bandit}
	Assume that at each iteration $t$, $\max_{k \in [NR]} r^t_k / r^t_{j_\star^t}$ $ \leq c(B,\varepsilon)$ for some constant $c$ that depends on $B$ and $\varepsilon$, then the iterate $\bm{\gamma}^t$ at the $t$-th iteration of the CD-Bernoulli algorithm (illustrated in Algorithm \ref{alg:Bandit}) for solving problem (\ref{eq:mle_r}) obeys
	\begin{align} 
		\mathbb{E}\left[\epsilon(\bm{\gamma}^{t}) \right] \leq  \frac{\alpha}{1+t-t_0},\label{eq:convergence_prop}
	\end{align}
	where
$
		\alpha^{-1}= \frac{1-\varepsilon}{(NR)^2c_1} + \frac{\varepsilon}{\eta^2c}$ with some constant $c_1>0$,	
for all $t\geq t_0 = \mathcal{O}(NR)$ and where  $\eta=\min_{k\in[NR]}	{\sum_{\ell}r_{\ell}^t}/{r_{k}^t}$ with $r_k^t$ defined in (\ref{eq:reward}). Furthermore, the CD-Random algorithm (illustrated in Algorithm \ref{tab:ML_coord}) for solving problem (\ref{eq:mle_r}) yields
	$
	\mathbb{E}\left[\epsilon(\bm{\gamma}^{t}) \right] \leq  \frac{  c_2(NR)^2}{NR+t},
$
	with some constant $c_2 >0$.
\end{theorem}
\begin{proof}
Please refer to Appendix \ref{proof_t1} for details.
\end{proof}
We conclude from Theorem \ref{thm:main_bandit} that by choosing proper values of $B$ and $\varepsilon$ (we use $B = NR/2$ and $\varepsilon = 0.6$ in the experiments of Section \ref{sec:simulaiton}) to yield sufficiently small $c(B,\varepsilon)$, the bound with respect to CD-Bernoulli approaches $\epsilon(\bm{\gamma}^t) = \mathcal{O}({\eta^2}/{t})$ with $\eta = \mathcal{O}(NR)$, which outperforms
the bound with respect to CD-Random, i.e.,  $\epsilon(\bm{\gamma}^t)= \mathcal{O}({(NR)^2}/{t})$. Hence, Theorem \ref{thm:main_bandit} demonstrates that for solving covariance-based joint device activity detection and data decoding, CD-Bernoulli yields a faster convergence rate than that with CD-Random.

In Algorithm \ref{alg:Bandit}, the value of $\varepsilon$ plays a vital role in the balance between exploitation and exploration. The larger the value of $\varepsilon$ is, the higher profitability of selecting the coordinate endowed with the largest current reward function $r_k^t$ (\ref{eq:reward}) at each iteration $t$ is. However, a larger value of $\varepsilon$ leads to insufficient exploration, which may lead to slow convergence rate. Instead of fixing $\varepsilon$, we prefer to developing a more flexible strategy for choosing $\varepsilon$. This motivates an improved algorithm to be presented in the next section.

\section{Coordiante Descent with Thompson sampling}\label{sec:4}
In this section, we improve the convergence rate of CD-Bernoulli Algorithm by incorporating another bandit problem to adaptively choose $\varepsilon$.  
Specifically, we formulate the choice of the parameter $\varepsilon$ as a general Bernoulli bandit problem, and develop a Thompson sampling algorithm for solving this bandit problem. The theoretical analysis is also presented to verify the advantage of Algorithm \ref{alg:Thompson} over Algorithm \ref{alg:Bandit}.

\subsection{A Stochastic MAP Problem for Choosing $\varepsilon$}
We first introduce a stochastic $q$-armed bandit problem for optimizing the parameter $\varepsilon$ in Algorithm \ref{alg:Bandit}.
In this paper, we assume that the reward distribution with respect to choosing $\varepsilon$ is \emph{Bernoulli}, i.e., the rewards are either $0$ or $1$. Note that the reward with respect to choosing $\varepsilon$ is different from the reward function of selecting the coordinates defined by (\ref{eq:reward}).

An algorithm for the MAB problem needs to decide which arm to play at each time step $t$, based on the outcomes of the previous $t-1$ plays. 
Let $\mu_i$ denote the (unknown) expected reward for arm $i$.
The means for the $q$-armed bandit problem, denoted as $\mu_1, \mu_2, \ldots, \mu_q$, are unknown, and are required to be learned by playing the corresponding arms.
A general way is to maximize the expected total reward by time $T$, i.e., $\Ex[\sum_{t=1}^{T} \mu_{i(t)}]$, where $i(t)$ is the arm played 
at step $t$, and the expectation is over the random choices of $i(t)$ made by the algorithm. 
The expected total \emph{regret} can be also represented as the loss that is generated due to not playing the optimal arm in each step. Let $\mu^{*}:=\max _{i} \mu_{i},$ and $d_{i}:=\mu^{*}-\mu_{i} .$ Also, let $k_{i}(t)$ denote the number
of times arm $i$ has been played up to step $t-1 .$ Then the expected total regret in time $T$ is given by \cite{agrawal2012analysis}
$
\mathbb{E}[\mathcal{R}(T)]=\mathbb{E}\left[\sum_{t=1}^{T}\left(\mu^{*}-\mu_{i(t)}\right)\right]=\sum_{i} d_{i} \cdot \mathbb{E}\left[k_{i}(T)\right].
$

\subsection{Thompson Sampling}
We first present some background on the Thompson sampling algorithm for the Bernoulli bandit problem, i.e., when the rewards are either $0$ or $1$, and for arm $i$ the probability of success (reward =$1$) is $\mu_i$. More details on Thompson sampling can be found in \cite{chapelle2011empirical} and \cite{agrawal2012analysis}. 

It is convenient to adopt Beta distribution as the Bayesian priors on the Bernoulli means $\mu_i$'s. 
Specifically, the probability density function (pdf) of $\Beta(\alpha, \beta)$, i.e., the beta distribution with parameters $\alpha > 0$, $\beta > 0$, is given by $f(x; \alpha, \beta) = \frac{\Gamma(\alpha+\beta)}{\Gamma(\alpha)\Gamma(\beta)}x^{\alpha-1}(1-x)^{\beta-1}$ with $\Gamma(\cdot)$ being the gamma function. If the prior is a $\Beta(\alpha, \beta)$ distribution, then based on a Bernoulli trial, the posterior distribution can be represented as $\Beta(\alpha+1, \beta)$ when the trail leads to a success; otherwise, it is updated as $\Beta(\alpha, \beta+1)$. 

The previous studies of Thompson sampling algorithm, e.g., \cite{agrawal2012analysis}, generally assumed that $\alpha$ and $\beta$ are integers. The algorithm initially assumes that arm $i$ has prior as $\Beta(1, 1)$ on $\mu_i$, which is
natural because $\Beta(1,1)$ is the uniform distribution on the interval $(0,1)$. At time $t$, having observed $S_i(t)$ successes (reward = $1$) and $F_i(t)$  failures (reward = $0$) in $k_i(t) = S_i(t)+F_i(t)$ plays of arm $i$, the algorithm updates the distribution on $\mu_i$ as $\Beta(S_i(t)+1, F_i(t)+1)$. 
The algorithm then samples from these posterior distributions of the $\mu_i$'s, and plays an arm according to the probability of its mean being the largest.

Different from previous methods, in this paper, we consider a more general way to update the parameters $\alpha$ and $\beta$ by evaluating the reward function $r_k^t$, to be presented in the following subsection.

\subsection{CD-Thompson}
\begin{algorithm}[tbp]
	\caption{CD-Thompson}\label{alg:Thompson}
	\begin{algorithmic}[1] 
		\STATE \textbf{Input: }  $E$. 
		\STATE \textbf{Initialize: }$\bm{\Sigma}=\sigma_n^2 \bm{ I}_{L}$, $\bm{\gamma}={\bm{0}}$,\\   set $\bar{r}^0_k = r^0_k$ for all $k \in [NR]$,
		\\the TS parameters $\bm{\alpha} = [\alpha_1,\cdots,\alpha_q]$ and with $\bm{\beta } = [ \beta_1,\cdots, \beta_q]$ some integer $q$.
		\FOR{$t=1$ {\bfseries to} $T$}
		\IF {$t \mod E == 0$}
		\STATE  set $\bar{r}_k^t = r^t_k$ for all $k \in [NR]$
		\ENDIF
		
		\STATE   For each arm $i=1,\cdots, q$, sample  ${{\nu}_{i}^t}\sim \mathrm{Beta}(\alpha_i,{ \beta}_i)$\label{L7}
		\STATE $j_t = \arg\max_{i}({{\nu}_{i}^t})$
		
		\STATE Generate $K \sim \mathrm{Bernoulli}({\nu}_{j_t}^t)$ 
		\IF {$K==1$}
		\STATE 	Select $k_t = \arg\max_{k\in [NR]} \bar{r}_{k}^{t}$
		\STATE  Compute $ \kappa_{k_t}^t = r_{k_t}^t/F(\bm{\gamma}^t)$ based on (\ref{eq:reward}).
		\STATE Update $\alpha_{j_t} = \alpha_{j_t} +{\nu}_{j_t}^t\cdot\kappa_{k_t}^t$
		
		\ELSE
		\STATE  Select $k_t \in [NR]$ uniformly at random
		
		\STATE  Compute $ \kappa_{k_t}^t=r_{k_t}^t/F(\bm{\gamma}^t)$ based on (\ref{eq:reward}).
		\STATE Update $ \beta_{j_t} =  \beta_{j_t} + (1-{\nu}_{j_t}^t)\kappa_{k_t}^t$
		
		\ENDIF\label{L18}
		
		\STATE	Update $\gamma_{k_t}^{t}$ according to the rule $\mathcal{H}_{k_t}$
		\STATE	Set $\bar{r}_{k_t}^{t+1} = r^{t+1}_{k_t}$ and $\bar{r}_{k}^{t+1} = \bar{r}_{k}^{t}$ for all $k\neq k_t$
		
		\ENDFOR
	\end{algorithmic}
\end{algorithm}
The coordinate descent algorithm via Thompson sampling (CD-Thompson) is illustrated in Algorithm \ref{alg:Thompson}. In this algorithm, a stochastic MAB problem for learning the best ${\nu}_{i}^t$ for arms $i = 1,\cdots,q$ at the $t$-th iteration is established, and a Thompson sampling algorithm is developed to solve this bandit problem. In Algorithm \ref{alg:Thompson}, the reward $r_k^t$ for selecting the $k$-th coordinate at time step $t$ is taken into consideration to update the parameters $\bm{\alpha}=[\alpha_1,\cdots,\alpha_q],\bm{\beta} = [ \beta_1,\cdots, \beta_q]$, thereby choosing ${{\nu}_{i}^t}$ based on ${{\nu}_{i}^t}\sim \mathrm{Beta}(\alpha_i,{ \beta}_i)$.
To be specific, for the index $j_t = \arg\max_{i}({{\nu}_{i}^t})$ and the Bernoulli variable $K\sim \mathrm{Bernoulli}({\nu}_{j_t}^t)$, if $K = 1$, we update 
\begin{align}\label{eq:alpha}
\alpha_{j_t} = \alpha_{j_t} +{\nu}_{j_t}^t\cdot r_{k_t}^t/F(\bm{\gamma}^t);
\end{align}
otherwise, we update 
\begin{align}\label{eq:beta}
\beta_{j_t} =  \beta_{j_t} + (1-{\nu}_{j_t}^t)r_{k_t}^t/F(\bm{\gamma}^t),
\end{align}
 where $r_{k_t}^t$ is defined in (\ref{eq:reward}) and $ F(\bm{\gamma}^t)$ is defined in (\ref{eq:mle_r}). For illustration, the main processes of CD-Bernoulli and CD-Thompson are illustrate in Fig. \ref{fig_alg}.
\begin{figure}[htbp]
	\centering
	\includegraphics[width=\columnwidth]{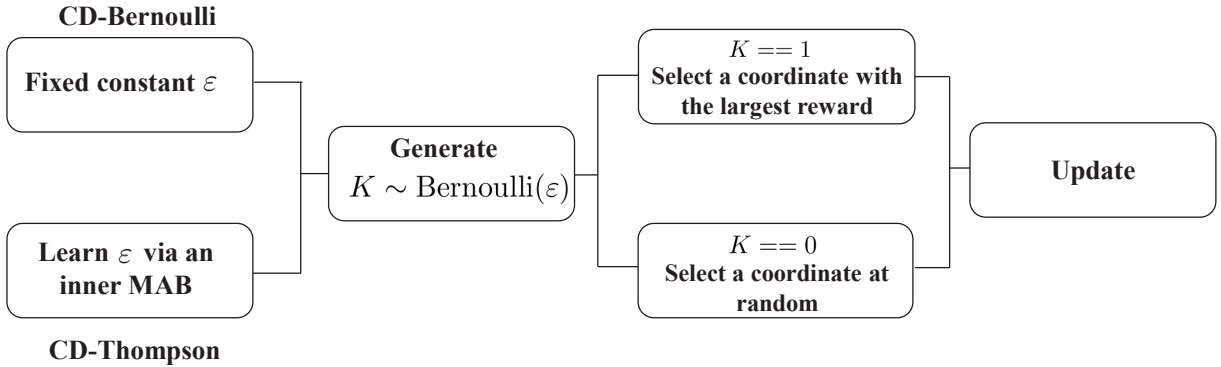}
	\caption{The main processes of CD-Bernoulli and CD-Thompson.}
	\label{fig_alg}
\end{figure}

 Recall that $\mu_i$ denotes the (unknown) expected reward for arm $i$. 
 At time $t$, if arm $i$ has been played a sufficient number of times, ${\nu}_{i}^t$ is tightly concentrated around $\mu_i$ with high probability. In the following analysis, we assume that the first arm is the unique optimal arm, i.e., $\mu_1 = \arg\max_{i\neq 1}\mu_i$. The expected regret for the stochastic MAB problem in Algorithm \ref{alg:Thompson} is presented as follows.
\begin{theorem}
	\label{th:N-arms}
The $q$-armed stochastic bandit problem for choosing ${\nu}_{i}^T$ for $i = 1,\cdots, q$ in Algorithm \ref{alg:Thompson} has an expected regret as \vspace{-0.1in}
	$\mathbb{E}[{\cal R}(T)] \leq \mathcal{O}\left(\left(\sum_{b=2}^q\frac{1}{d_b^2}\right)^2 \ln T\right)  \vspace{-0.1in}$
	in time $T$, where $d_i=\mu_1- \mu_i$.
\end{theorem}
\begin{proof}
Please refer to Appendix \ref{sec:proof_bandit} for a brief summary of the proof.
\end{proof}

\begin{remark}
 Algorithm \ref{alg:Bandit} adopts a fixed constant $\varepsilon>0$ as the probability of updating coordinate $k$ with the largest reward function (i.e., coordinate-wise descent value) $r_k^t$ (\ref{eq:reward}) at time step $t$, which lacks flexibility for better exploration-exploitation trade-off. In contrast, Algorithm \ref{alg:Thompson} improves the strategy of choosing the parameter $\varepsilon$ in Algorithm \ref{alg:Bandit}. This is achieved by establishing a stochastic $q$-armed bandit problem for choosing the corresponding probability. 
This multi-armed stochastic bandit problem studies an exploitation/exploration trade-off by sequentially designing ${\nu}_{i}^T$ for $i = 1,\cdots, q$ at time step $t$. During the sequential decision, Algorithm \ref{alg:Thompson} is able to approximate the optimal value of the probability. Theoretically, Theorem \ref{th:N-arms}	demonstrates that Algorithm \ref{alg:Thompson} enjoys a logarithmic expected regret for the stochastic $q$-armed bandit problem, which typically is the best to expect. Furthermore, the exploitation/exploration trade-off in Algorithm \ref{alg:Thompson} eludes the situation where the large value of ${\nu}_{i}^t$ in Algorithm \ref{alg:Thompson} is maintained in many time steps, and thus avoids high computational cost for computing $r_k^t$ for all $k\in[NR]$ at time step $t$.
\end{remark}
\begin{remark}
	Different from the previous MAB based coordinate descent algorithm \cite{salehi2018coordinate} that solves \emph{convex} optimization problems, our proposed algorithm solves a covariance-based estimation problem that is \emph{non-convex}.  Beta distribution, i.e., $\Beta(\alpha, \beta)$, is a powerful tool to learn 
	the priors for Bernoulli rewards. 
	Specifically, we consider a more general way to update the parameters $\alpha$ and $\beta$ based on the reward function $r_k^t$. Our proposed algorithms turn out to be enjoying faster convergence rates with modest computational time complexity.
\end{remark}
\section{Application to Massive Connectivity with Low-precision ADCs}
While the formulation in Section \ref{sys} presents a basic massive connectivity system, the proposed algorithms, i.e., CD-Bernoulli and CD-Thompson, can also be applied to solve more general activity detection problems. In this section, we introduce massive connectivity with low-precision analog-to-digital converters (ADCs) as an example. Recently, the use of low precision (e.g., 1--4 bits) ADCs in massive MIMO systems has been proposed to reduce cost and power
consumption \cite{mo2017channel,7894211,wen2015bayes}. In the following, we illustrate how the proposed algorithms can be applied to this new scenario.

At each of receive antennas, the A/D converter samples the
received signal and utilizes a finite number of bits to represent
corresponding samples. 
Each entry, i.e., $Y_{i j}$, of $\bm{Y}$ (\ref{eq:sys_model}) for $1 \leq i \leq L, 1 \leq j \leq M$ is quantized into a finite set of pre-defined values by a $b$-bit
quantizer $\mathrm{Q}_{c}$. The quantized received signal is thus represented by \cite{wen2015bayes}
\begin{align}
{\bm{Y}}_{\mathrm{q}}=\mathrm{Q}_{c}(\bm{Y}) = \mathrm{Q}_c(\bm{Q}\bm{\Gamma}^{\frac{1}{2}}\bm{H}+\bm{N}),
\end{align}
where the complex-valued quantizer $\mathrm{Q}_{c}(\cdot)$ is defined
as $X_{\mathrm{q}}=\mathrm{Q}_{c}\left(X\right) \triangleq \mathrm{Q}\left(\operatorname{Re}\left\{X\right\}\right)+\mathrm{i} \mathrm{Q}\left(\operatorname{Im}\left\{X\right\}\right),$ i.e., the
real and imaginary parts are quantized separately. The real
valued quantizer $\mathrm{Q}$ maps a real-valued input to one of the $2^{b}$ bins, which are characterized by the set of $2^{{b}}-1$ thresholds
$\left[r_{1}, r_{2}, \ldots, r_{2^{b}-1}\right],$ such that $-\infty<r_{1}<r_{2}<r_{2}<\cdots<\infty .$ For $z=1, \ldots, 2^{b}-1$, an element of the output ${\bm{Y}}_{\mathrm{q}}$ is assigned a value in $\left(r_{z-1}, r_{z}\right]$
when the quantizer entry of the input $\bm{Y}$ falls in the $z$-th bin, i.e., the interval $\left(r_{z-1}, r_{z}\right]$.

Generally, the quantization operation is nonlinear. For ease of applying coordinate descent algorithms to solve quantized model, we linearize the quantizer. Based on Bussgang's theorem, the quantizer output ${\bm{Y}}_{\mathrm{q}}$
can be decomposed into a signal component plus a distortion
$\bm{W}_{\mathrm{q}}\in\mathbb{C}^{L\times M}$ that is uncorrelated with the signal component $\bm{Y}$ \cite{mo2017channel}, i.e.,
\begin{align}\label{eq:q}
{\bm{Y}}_{\mathrm{q}} = \left(\bm{I}_{M}-\bm{\rho}\right)\bm{Y} + \bm{W}_{\mathrm{q}},
\end{align}
where $\bm{\rho}$ is the real-valued
diagonal matrix containing the $M$ distortion factors:
\begin{equation}
\bm{\rho}=\left[\begin{array}{ccc}{\rho_{1}} & {} & {} \\ {} & {\ddots} & {} \\ {} & {} & {\rho_{M}}\end{array}\right] \approx\left[\begin{array}{ccc}{2^{-2 b_{1}}} & {} & {} \\ {} & {\ddots} & {} \\ {} & {} & {2^{-2 b_{M}}}\end{array}\right],
\end{equation}
with $b_j$ for $j= 1,\cdots, M$ denoting the bit resolution of the scalar quantizer with respect to each antenna.

Since $\bm{W}_{\mathrm{q}}$ is uncorrelated with the signal component $\bm{Y}$, the covariance matrix of the quantizer can be represented as
\begin{equation}
{\bm{\Sigma}}_{\mathrm{q}}=\mathrm{E}\left[{\bm{Y}}_{\mathrm{q}} {\bm{Y}}_{\mathrm{q}}^{\mathsf{H}}\right]=\bm{\rho} \bm{\Sigma} \bm{\rho}+\bm{\rho}\left(\bm{I}_{M}-\bm{\rho}\right) \operatorname{diag}\left(\bm{\Sigma}\right),
\end{equation}
where $\bm{\Sigma}$ is defined in (\ref{eq:distri}). Hence, the joint device activity detection and data decoding with low-precision ADCs can be formulated as
\begin{align}
\label{eq:mleadc}
\mathop{{\text{minimize}}}\limits_{{{\bm\gamma }} \geq 0} &\quad F(\bm\gamma):=\log |{\bm{\Sigma }_{\mathrm{q}}}| + \frac{1}{M}{\mathrm{Tr}}\left( {{{\bm{\Sigma }_{\mathrm{q}}}^{ - 1}}{\bm{Y}_{\mathrm{q}}}{{\bm{Y}}_{\mathrm{q}}^{\mathsf{H}}}} \right).
\end{align}
Problem (\ref{eq:mleadc}) can be efficiently solved by the proposed algorithms, i.e., Algorithm \ref{alg:Bandit} and Algorithm \ref{alg:Thompson}. Simulations will be presented in the next section. 
\section{Simulation Results}\label{sec:simulaiton}\label{sec:5}
In this section, we provide simulation results to demonstrate that the proposed algorithms enjoy faster convergence rates than coordinate descent with random sampling for joint device activity detection and data decoding. Furthermore, we apply our proposed algorithms to massive connectivity with low-precision ADCs.
\subsection{Simulation Settings and Performance Metric} 
Consider a single cell of radius $1000$m containing $N =
1500$ devices, among which $K = 50$ devices are active. 
The performance is characterized by the probability of missed detection. 

The simulation settings are given as follows:
\begin{itemize}
	\item The signature matrix $\bm{Q}\in\mathbb{C}^{L\times NR}$ (\ref{eq:Q}) with $R =2^J$ is generated from i.i.d. standard complex Gaussian distribution, followed by normalization, i.e., 
	\[
	\bm{Q}\sim \mathcal{N}(\bm{0},\frac{1}{2L}\bm{I}_{L})+\mathrm{i}\mathcal{N}(\bm{0},\frac{1}{2L}\bm{I}_{L}).
	\]
	
	\item The channel matrix $\bm{H}\in\mathbb{C}^{NR\times M}$ consists of Rayleigh fading components that follow i.i.d. standard complex Gaussian distribution, i.e.,
	\[
	\bm{H}\sim \mathcal{N}(\bm{0},\frac{1}{2}\bm{I}_{NR})+\mathrm{i}\mathcal{N}(\bm{0},\frac{1}{2}\bm{I}_{NR}).
	\] Meanwhile, the fading component $g_i$ in (\ref{eq:hg}) for device $i$ with $i = 1,\cdots, N$ is given as $g_i = -128.1-37.6\log_{10}(d_i)$ in dB where $d_i = 1000$ for all $\forall i\in[N]$.
	\item The additive noise matrix $\bm{N}\in\mathbb{C}^{L\times M}$ is generated from i.i.d. complex Gaussian distribution, i.e.,
		\[
	\bm{N}\sim \mathcal{N}(\bm{0},\frac{1}{2\sigma_n^2}\bm{I}_{L})+\mathrm{i}\mathcal{N}(\bm{0},\frac{1}{2\sigma_n^2}\bm{I}_{L}),
	\]where the variance $\sigma_n^2$ is the background noise power normalized by the device transmit
	power.
In the simulations, the background noise power is set as -$99$ dBm,
	and the transmit power of each device is set as $40$ dBm.
	\item Performance metric is defined in the following. The missed detection occurs when a
	device is active but is detected to be inactive, or a device
	is active and is detected to be active but the data decoding is incorrect.
	Different probabilities of missed detection can
	be obtained by adjusting the value of the threshold $s_{t h}$ in (\ref{eq:thresh}). In the simulations, we choose a threshold $s_{t h}$ that enables to determine $50$ active devices from the estimated $\hat{\bm{\gamma}}$.
	
\end{itemize}

The following three algorithms are compared:
\begin{itemize}
	\item \textbf{Proposed coordinate descent with Bernoulli sampling (CD-Bernoulli)}: Problem (\ref{eq:mle_r}) is solved by Algorithm \ref{alg:Bandit} with the setting of $B = NR/2$ and $\varepsilon = 0.6$. Note that the computational time will increase as the value of $\varepsilon$ increases. The convergence rate of CD-Bernoulli will decrease as the value of $\varepsilon$ decreases. We thus pick a modest value to illustrate the performance of CD-Bernoulli.
	\item \textbf{Proposed coordinate descent with Thompson sampling (CD-Thompson)}: Problem (\ref{eq:mle_r}) is solved by Algorithm \ref{alg:Thompson} with the setting of $B = NR/2$ and $q=10$.
	\item \textbf{Coordinate descent with random sampling (CD-Random)}:  Problem (\ref{eq:mle_r}) is solved by Algorithm \ref{tab:ML_coord} with uniformly randomly choosing a coordinate to update.
\end{itemize}

All the algorithms stop when the relative change of the objective function $F(\bm{\gamma}^t)$ is lower than a certain level, i.e.,
\[\frac{|F(\bm{\gamma}^{t+1}) - F(\bm{\gamma}^t)|}{|F(\bm{\gamma}^t)|}\leq 10^{-6}
\] or the number of iterations exceeds $1500$.
\subsection{Convergence Rate}\label{sec:con}
In the simulations, the length of the signature sequences is $L=300$, the number of antenna is $M = 16$, and each device transmits a message of $J=1$ bit or $J=2$ bits.
The convergence rates of different algorithms are illustrated in Fig. \ref{fig1}.
We validate the convergence rate analysis in Theorem \ref{thm:main_bandit} by comparing CD-Bernoulli (i.e., Algorithm \ref{alg:Bandit}) with CD-Random (i.e., Algorithm \ref{tab:ML_coord}). Furthermore, Fig. \ref{fig1} shows that CD-Thompson with a more sophisticated strategy on choosing the probability of updating the coordinate has better performance than Algorithm \ref{alg:Bandit}. As illustrated in Fig. \ref{fig1} and demonstrated in Theorem \ref{thm:main_bandit}, a larger value of $J$ yields a large value of $NR$, which leads to a slower convergence rate.  In summary, this simulation shows that the proposed algorithms yield faster convergence rates than the state-of-the-art algorithm \cite{8761672}. 
\begin{figure}[t]
	\centering
	\includegraphics[width=0.4\columnwidth]{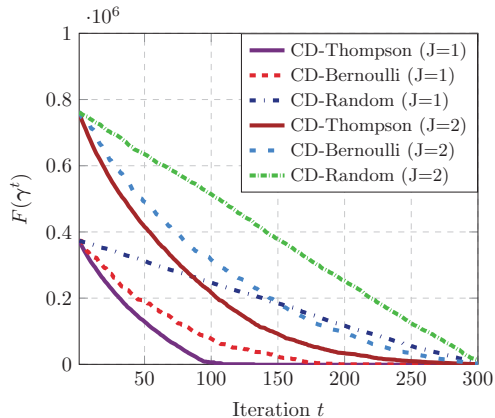}
	\caption{Convergence rates of coordinate descent with respect to three coordinate selection strategies.}
	\label{fig1}
\end{figure}
\subsection{Probability of Missed Detection}\label{sec:c}
Under the setting of  $L = 200, J=1,M = 16$, the computational time of three algorithms is further illustrated in Fig. \ref{fig2}. It shows that the proposed algorithms achieve the same level of detection accuracy with much less computational time than the algorithm in \cite{8761672}. The reason is that the coordinate selection with Bernoulli sampling or Thompson sampling is able to choose the coordinate that yields a larger descent in the objective value. Additionally, Fig. \ref{fig2} also shows that Algorithm \ref{alg:Thompson} can further reduce the computational time, compared to Algorithm \ref{alg:Bandit}. This is achieved by a better exploitation/exploration trade-off in Algorithm \ref{alg:Thompson} which eludes the situation where the large value of ${\nu}_{i}^t$ in Algorithm \ref{alg:Thompson} is maintained in many time steps, which leads to a high computational cost for computing $r_k^t$ for all $k\in[NR]$ in time step $t$.

\begin{figure}[t]
	\centering
	\includegraphics[width=0.4\columnwidth]{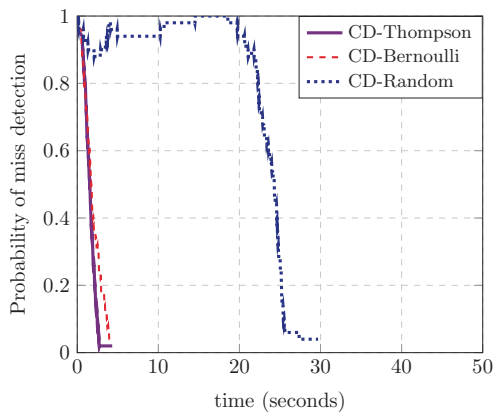}
	\caption{Probability of missed detection vs. computational time.}
	\label{fig2}
\end{figure}
\subsection{Applications in Low-precision ADCs}
In this part, we test the proposed algorithms with low-precision ADCs. For the quantization procedure, we use the typical uniform quantizer with the quantization step-size $s_{\mathrm{q}}= 0.5$. For $b$-bit quantization, the threshold of this uniform quantizer is given by 
\begin{align}
r_z = (-2^{b-1}+ z)s_{\mathrm{q}},\quad \text{for}~z=1, \ldots, 2^{b}-1,
\end{align}
 and the element of the quantization output $\bm{Y}_{\mathrm{q}}$ (\ref{eq:q}) is assigned the value $r_z -\frac{s_{\mathrm{q}}}{2}$ when
the input falls in the $z$-th bin, i.e., $\left(r_{z-1}, r_{z}\right]$.

Under the same setting as Section \ref{sec:c}, Fig. \ref{fig3} shows the unquantization case, and the quantization case with different quantization levels, i.e., $b = \{1,2,3\}$. To further illustrate the computational cost of the proposed algorithm applied to the low-precision ADCs, Fig. \ref{fig4} shows the probability of missed detection with respect to computational time. These results demonstrate that $3$-bit quantization is sufficient to achieve similar convergence rate and accuracy as the unquantization scenario.
\begin{figure}[t]
	\centering
	\includegraphics[width=0.4\columnwidth]{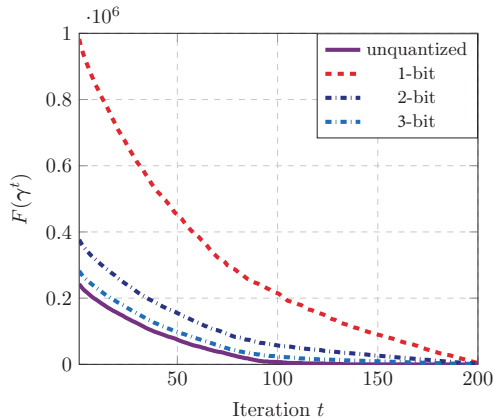}
	\caption{Convergence rates of coordinate descent with Thompson sampling for massive connectivity with low-precision ADCs.}
	\label{fig3}
\end{figure}
\begin{figure}[t]
	\centering
	\includegraphics[width=0.4\columnwidth]{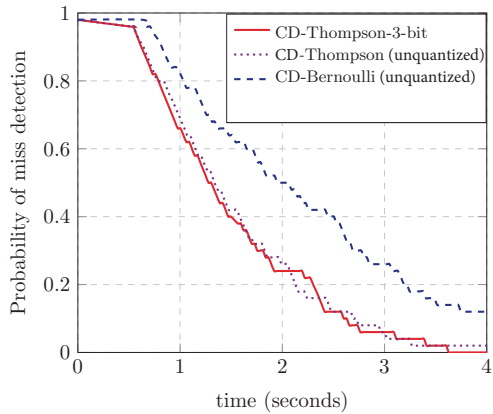} 
	\caption{Probability of missed detection.}
	\label{fig4}
\end{figure}
\section{Conclusions}
In this paper, we developed efficient algorithms based on multi-armed bandit to solve the joint device activity detection and data decoding problem in massive random access. Specifically, we exploited a multi-armed
bandit algorithm to learn to update the coordinate, thereby resulting in more aggressive descent of the objective function. To further improve the convergence rate, an inner multi-armed bandit problem was established to improve the exploration policy. The performance gains in the convergence rate and time complexity of the proposed algorithms over the start-of-the-art algorithm were demonstrated both theoretically and empirically. Furthermore, our proposed algorithms can be applied to a more general scenario, i.e., activity and data detection in the low-precision analog-to-digital
converters (ADCs), thereby saving energy and reducing the power consumption.

Our proposed algorithm only updates a single coordinate at each time step $t$. It is interesting to further investigate the effect of 
choosing multiple coordinates from a budget at each time step. At a high level, the proposed approach can be regarded as an instance of ``learning to optimize'', i.e., applying machine learning to solve optimization problems. Specifically, it belongs to optimization policy learning \cite{bengio2018machine}, which learns a specific policy for some optimization algorithm. One related work is \cite{8879693}, which learns the pruning policy of the branch-and-bound algorithm. It is interesting to apply such an approach to other optimization algorithms to improve the computational efficiency for massive connectivity.

\appendices
\section{Computation of the Reward Function}\label{sec:reward}
In this section, we derive the reward function for the multiple-armed bandit problem for coordinate descent. Define $k \in [N]$ as the index of the selected coordinate and define 
$F_k(d)=F(\bm{\gamma}+ d \bm{e}_k)$ where $\bm{e}_k$ denotes the $k$-th canonical basis with a single $1$ at its $k$-th coordinate and zeros elsewhere.
We can simplify $F_k(d)$ as follows
\begin{align}
F_k(d)=& \log \big | \bm{\Sigma} \big | +\frac{1}{M}{\mathrm{Tr}}\left( {{{\bm{\Sigma }}^{ - 1}}{\bm{Y}}{{\bm{Y}}^{\mathsf{H}}}} \right)+ \log (1+ d\, \bm{a}_k^{\mathsf{H}} \bm{\Sigma} ^{-1} \bm{a}_k)  -  \frac{  \bm{a}_k^{\mathsf{H}} \bm{\Sigma}^{-1} \widehat{\bm{\Sigma}}_{\bm{y}} \bm{\Sigma}^{-1} \bm{a}_k }{1+ d \, \bm{a}_k^{\mathsf{H}} \bm{\Sigma}^{-1} \bm{a}_k} d. \label{dumm_ml_bpp3}
\end{align}
According to \cite{haghighatshoar2018improved}, the global minimum of $F_k(d)$ in $(-\frac{1}{\bm{a}_k^{\mathsf{H}} \bm{\Sigma}^{-1}\bm{a}_k},+\infty)$ is
$
\delta= \frac{ \bm{a}_k^{\mathsf{H}} \bm{\Sigma}^{-1} \widehat{\bm{\Sigma}}_{\bm{y}} \bm{\Sigma}^{-1} \bm{a}_k -  \bm{a}_k^{\mathsf{H}} \bm{\Sigma}^{-1}\bm{a}_k }{(\bm{a}_k^{\mathsf{H}} \bm{\Sigma}^{-1}\bm{a}_k )^2},
$
 so the descent value of the cost function $F(\bm{\gamma})$ is:
\begin{align}
 F(\bm{\gamma}) - F_k(\delta)&=F(\bm{\gamma}) -F(\bm{\gamma}+\delta) =\frac{\bm{a}_{k}^{\mathsf{H}} \boldsymbol{\Sigma}^{-1} \widehat{\boldsymbol{\Sigma}}_{\bm{y}} \boldsymbol{\Sigma}^{-1} \bm{a}_{k}}{1+\delta \bm{a}_{k}^{\mathsf{H}} \boldsymbol{\Sigma}^{-1} \bm{a}_{k}} \delta-\log\left(1+\delta \bm{a}_{k}^{\mathsf{H}} \boldsymbol{\Sigma}^{-1} \bm{a}_{k}\right).
\end{align}
Hence, the reward function $r_k$ is defined as
$
r_k =  F(\bm{\gamma}) - F_k(\delta).
$

\section{Primary theorems for the Proof of Theorem \ref{thm:main_bandit}}
Several theorems are needed to pave the way for the proof of Theorem \ref{thm:main_bandit}. 
\begin{theorem} \label{thm:linear}
	Recall the reward function $r_k$ defined in \eqref{eq:reward}.
	Under the assumptions of Lemma~\ref{lemma:1}, if we choose the coordinate $k$ with the largest $r_k^t$ at the $t$-th iteration, it yields the following linear convergence guarantee: 		
	\begin{equation} \label{eq:convergence_gaurantee}
	\epsilon(\bm{\gamma}^{t}) \leq
	\epsilon(\bm \gamma^0) \prod_{j=1}^t \left(1 -  \max_{k\in [d]} \frac{r_k^t}{{\sum_{\ell}r_{\ell}^t}} \right),
	\end{equation}
	for all $t>0$, where $\epsilon(\bm{\gamma}^{0})$ is the sub-optimality gap at  $t=0$. 
\end{theorem}
\begin{proof}
Please refer to Appendix \ref{proof_thm:linear} for details.
\end{proof}

\begin{theorem} \label{thm:main}
	Under the assumptions of Lemma~\ref{lemma:1}, we have the following convergence guarantee:
	\begin{equation} \label{eq:convergence}
	\epsilon(\bm{\gamma}^{t}) \leq  \frac{ \eta^2}{NR+t - t_0}
	\end{equation}
	for all $t\geq t_0$, where $t_0 = \max \{1,NR \log \frac{NR \epsilon(\bm{\gamma}^0)}{\eta^2} \}$, $\epsilon(\bm{\gamma}^{0})$ is the sub-optimality gap at  $t=0$ and $\eta = O(NR)$ is an upper bound on
	$\min_{k\in[NR]} \frac{\sum_{\ell}r_{\ell}^t}{r_{k}^t}$ for all iterations $j \in [t]$.
\end{theorem}

\section{Proof of Theorem \ref{thm:linear}}\label{proof_thm:linear}
The selection strategy concerned in this proof is to choose the coordinated $k$ with the largest reward function $r_k^t$ defined in (\ref{eq:reward}), which is denoted by $k^{\star}$.
Hence, based on the fact $\sum_{\ell}r_{\ell}^t\geq \epsilon(\bm{\gamma}^t)$ it yields that
\begin{align} \label{eq:temp_l3}
\epsilon(\bm{\gamma}^{t+1}) - \epsilon(\bm{\gamma}^t) 
= F(\bm{\gamma}^{t+1}) - F(\bm{\gamma}^t)  \leq - r_{{k^{\star}}}^t 
-\sum_{\ell}r_{\ell}^t \max_{k\in [NR]} \frac{r_k^t }{\sum_{\ell}r_{\ell}^t } \leq - \epsilon(\bm{\gamma}^t) \max_{k\in [NR]} \frac{r_k^t }{\sum_{\ell}r_{\ell}^t },
\end{align}
that induces
\begin{align} \label{eq:temp_l4}
\epsilon(\bm{\gamma}^{t+1}) \leq
\epsilon(\bm{\gamma}^t) - \epsilon(\bm{\gamma}^t) \max_{k\in [NR]} \frac{r_k^t }{\sum_{\ell}r_{\ell}^t },
\end{align}
which leads to
\begin{align} \label{eq:temp_l5}
\epsilon(\bm{\gamma}^{t+1}) \leq
\epsilon(\bm{\gamma}^t) \left( 1 - \max_{k\in [NR]} \frac{r_k^t }{\sum_{\ell}r_{\ell}^t } \right).
\end{align}

\section{Proof of Theorem \ref{thm:main}}

According to $F(\bm{\gamma}^{t+1})- F(\bm{\gamma}^t) = \epsilon(\bm{\gamma}^{t+1})- \epsilon(\bm{\gamma}^t)$, we get
$
\epsilon(\bm{\gamma}^{t+1}) - \epsilon(\bm{\gamma}^t)
\leq -{r_{k^\star}^t}.
$

As ${k^\star}$ is the coordinate with the largest $r_{k}^t$, we have
\begin{equation} \label{eq:temp}
\epsilon(\bm{\gamma}^{t+1})  - \epsilon(\bm{\gamma}^t)
\leq -{r_{k^\star}(\bm{\gamma}^t)} \leq -\frac{\sum_{\ell}r_{\ell}^t}{NR}.
\end{equation}

According to the definition of $\epsilon(\bm{\gamma}^t)$ (\ref{epsilon}) and the coordinate-wise reward function (\ref{eq:reward}), we have $\epsilon(\bm{\gamma}^t) \leq \sum_{\ell=1}^{NR}r_{\ell}^t$. Plugging the inequality $\epsilon(\bm{\gamma}^t) \leq \sum_{\ell=1}^{NR}r_{\ell}^t$ in \eqref{eq:temp} yields
\begin{align} \label{eq:1_bound}
\epsilon(\bm{\gamma}^{t+1}) - \epsilon(\bm{\gamma}^t)\leq -\frac{\sum_{\ell}r_{\ell}^t}{NR} \leq -\frac{\epsilon(\bm{\gamma}^t)}{NR},
\end{align}
thus, it arrives
\begin{equation} \label{eq:decrease_s1}
\epsilon(\bm{\gamma}^{t+1}) \leq \epsilon(\bm{\gamma}^t) \cdot \left(1-\frac{1}{NR}\right).
\end{equation}

Furthermore, the inductive step at time $j+1$ is justified by plugging \eqref{eq:convergence} in \eqref{eq:decrease_s1}:
\begin{align}
\begin{aligned}
\epsilon(\bm{\gamma}^{t+1}) &\leq \frac{\eta^2}{NR + t - t_0} \left(1-\frac{1}{NR}\right) \leq \frac{\eta^2}{NR+t+1-t_0}.
\end{aligned}
\end{align}

To complete the proof, the induction base case for $t = t_0$ needs to be justified, i.e., we need to show that \begin{equation} \label{eq:basis}
\epsilon(\bm \gamma^{t_0}) \leq \frac{\eta^2}{  NR}.
\end{equation}
%
%
The proof based on the contradiction is used to identify the induction base, that is, assuming $\epsilon(\bm{\gamma}^{t_0}) > \frac{ \eta^2}{  NR}$ leads to a contradiction.
If $\epsilon(\bm{\gamma}^{t_0}) > \frac{ \eta^2}{  NR}$, then 
\begin{equation} \label{eq:contradic}
\frac{1}{NR} < \frac{\epsilon(\bm{\gamma}^{t_0}) }{ \eta^2}. 
\end{equation} 
Based on (\ref{eq:decrease_s1}), there is
\begin{equation}
\epsilon(\bm{\gamma}^{t_0}) \leq \epsilon(\bm{\gamma}^{0}) \left(1-\frac{1}{NR} \right)^{t_0}.
\end{equation}
Based on the inequality such that $1 + x < \exp(x)$ for $x<1$ we have
\begin{align*}
\epsilon(\bm{\gamma}^{t_0}) \leq \epsilon(\bm{\gamma}^{0}) \exp(-\frac{t_0}{NR}) \leq \epsilon(\bm{\gamma}^{0}) \exp(-\log \frac{NR  \cdot \epsilon(\bm{\gamma}^0)}{ \eta^2}) 
= \epsilon(\bm{\gamma}^{0}) \frac{ \eta^2}{  NR\cdot \epsilon(\bm{\gamma}^0)} = \frac{ \eta^2}{  NR},
\end{align*}
which yields a contradiction with respect to  the assumption $\epsilon(\bm{\gamma}^{t_0}) > \frac{ \eta^2}{  NR}$. It thus shows that the induction base holds and completes the proof.

\section{proof of Theorem \ref{thm:main_bandit}}\label{proof_t1}
We first consider the iterate $\bm{\gamma}^t$ at the $t$-th iteration of the coordinate descent with Bernoulli sampling (illustrated in Algorithm \ref{alg:Bandit}).
Suppose that \eqref{eq:convergence} holds for some $t \geq t_0$. We shall verify it for $t+1$. 
We start the analysis by computing the expected marginal decrease for $\varepsilon$ in Algorithm~\ref{alg:Bandit}, 
\begin{align} \label{eq:decrease1}
\mathbb{E} \left[r_k^t|\bm{\gamma}^t\right] \geq (1-\varepsilon)\frac{1}{c_1\cdot NR}   r_k^t +  \varepsilon \frac{r^t_{{k^{\star}}}}{c},
\end{align}
where $c_1>0$ is some finite constant and $c$ is a finite constant defined in Theorem~\ref{thm:main_bandit} and ${k^{\star}}= \arg\max_{k \in [NR]} r_k^t$. The expectation is with respect to the random choice of the algorithm.

For all $k \in [NR]$,  it holds
\begin{align} 
\mathbb{E} \left[r_k^t|\bm{\gamma}^t\right] \geq (1-\varepsilon)\frac{1}{c_1\cdot NR} \left( \sum_{\ell=1}^{NR} \frac{(r_\ell^t)^2  }{ NR}\right) + \varepsilon
 \frac{(r_{{k^{\star}}}^t)^2}{ c}  \geq
  (1-\varepsilon) \frac{\left(\sum_{\ell=1}^{NR} r_{\ell}^t\right)^2}{(NR)^2c_1} + \varepsilon \frac{\left(\sum_{\ell=1}^{NR} r_{\ell}^t\right)^2}{\eta^2c}  \label{eq:decrease3},
\end{align}
where \eqref{eq:decrease3} follows from the assumption $\sum_{\ell}r_{\ell}^t \leq \eta r_{{k^{\star}}}^t$ in Theorem~\ref{thm:main_bandit}.
Plugging the inequality $\epsilon(\bm{\gamma}^t) < \sum_{\ell}r_{\ell}^t$ in \eqref{eq:decrease3}, it yields
\begin{align} 
\mathbb{E} \left[r_k^t|\bm{\gamma}^t\right]  \geq
 \epsilon^2(\bm{\gamma}^t) \left(  \frac{1-\varepsilon}{(NR)^2c_1} + \frac{\varepsilon}{\eta^2c} \right) = 
\frac{\epsilon^2(\bm{\gamma}^t)}{\alpha}
. \label{eq:second}
\end{align}

Then, based on \eqref{eq:second}, the induction hypothesis is scrutinized by
\begin{align}\label{eq:decrease4}
\mathbb{E} [\epsilon(\bm{\gamma}^{t+1})] - \mathbb{E}[\epsilon(\bm{\gamma}^{t})]
 \leq ~ \mathbb{E}\left[r_k^t|\bm{\gamma}^t\right]   \leq 
-\mathbb{E}\left[ \frac{\epsilon^2(\bm{\gamma}^t)}{\alpha}  \right]\leq-\frac{\mathbb{E}[\epsilon(\bm{\gamma}^t)]^2}{\alpha},
\end{align}
where the last inequality is based on the Jensen's inequality (i.e., $\mathbb{E}[\epsilon(\bm{\gamma}^t)]^2 \leq \mathbb{E}[\epsilon^2(\bm{\gamma}^t)]$). By reformulating the terms in \eqref{eq:decrease4} we get
\begin{align} \label{eq:decrease6}
\mathbb{E}[\epsilon(\bm{\gamma}^{t+1})] &\leq 
\mathbb{E}\left[\epsilon(\bm{\gamma}^{t}) \right] \left(1- \frac{\mathbb{E}\left[\epsilon(\bm{\gamma}^{t}) \right]}{\alpha} \right).
\end{align}

Let $f(x) = x\left(1-\frac{x}{\alpha}\right)$, as $f'(x)>0$ for $x< \frac{\alpha}{2}$, and plugging \eqref{eq:convergence_prop} in \eqref{eq:decrease6}, it leads to the inductive step at time $t+1$:
\begin{align} 
\mathbb{E}[\epsilon(\bm{\gamma}^{t+1})]
\leq 
\mathbb{E}\left[\epsilon(\bm{\gamma}^{t}) \right] \left(1- \frac{\mathbb{E}\left[\epsilon(\bm{\gamma}^{t}) \right]}{\alpha} \right)
 \leq
\frac{\alpha}{1+t-t_0} \cdot \left(1-\frac{1}{1+t-t_0} \right) \leq 	\frac{\alpha}{1+t+1-t_0}.
\label{eq:temp4}
\end{align}

We are left to show that the induction basis is satisfied.
By using the inequality \eqref{eq:decrease4} for $t=1,\ldots,t_0$ we get 
\begin{align}\label{eq:decrease7}
\mathbb{E}[\epsilon(\bm{\gamma}^{t_0})] &\leq \epsilon(\bm{\gamma}^{0}) 
- \sum_{t=0}^{t_0-1}\frac{\mathbb{E}[\epsilon(\bm{\gamma}^t)]^2}{\alpha} .
\end{align}

Since at each iteration the cost function decreases, we have $\epsilon(\bm{\gamma}^{t+1})\leq \epsilon(\bm{\gamma}^{t})$ for all $t\geq 0$. 
Hence, if  $\mathbb{E}[\epsilon(\bm{\gamma}^{t})]\leq \frac{\alpha}{2}$ for each $0\leq t \leq t_0$, it concludes that $\mathbb{E}[\epsilon(\bm{\gamma}^{t_0})]\leq \frac{\alpha}{2}$. 
The induction hypothesis is justified via showing that $\mathbb{E}[\epsilon(\bm{\gamma}^{t_0})] > \frac{\alpha}{2}$ results in a contradiction.
Under this assumption, \eqref{eq:decrease7} is reformulated as
\begin{align}\label{eq:decrease8}
\mathbb{E}[\epsilon(\bm{\gamma}^{t_0})] &\leq 
\epsilon(\bm{\gamma}^{0}) - t_0 \frac{\alpha}{2} 
= \epsilon(\bm{\gamma}^{0}) \left( 1- t_0 \frac{\alpha}{2\epsilon(\bm{\gamma}^{0})} \right).
\end{align}
Furthermore, based on the inequality $1+x\leq \exp(x)$ with \eqref{eq:decrease8}, we get
\begin{align}\label{eq:decrease9}
E[\epsilon(\bm{\gamma}^{t_0})] &\leq 
\epsilon(\bm{\gamma}^{0}) \exp\left( -t_0 \frac{\alpha}{2\epsilon(\bm{\gamma}^{0})} \right).
\end{align}
We plug $t_0 = \frac{2 \epsilon(\bm{\gamma}^{0})}{\alpha} \log(\frac{ \epsilon(\bm{\gamma}^{0})}{\alpha})$
in \eqref{eq:decrease9} to get
$
E[\epsilon(\bm{\gamma}^{t_0})] \leq {\alpha} ,
$
which completes the proof.

Then, we focus on the analysis of the iterate $\bm{\gamma}^t$ at the $t$-th iteration of the coordinate descent with random sampling (illustrated in Algorithm \ref{tab:ML_coord}).
Suppose that \eqref{eq:convergence} holds for some $t \geq t_0$. We want to verify it for $t+1$. 
The analysis is begin with computing the expected marginal decrease for $\varepsilon$ in Algorithm \ref{tab:ML_coord}. For some constant $c_2>0$, there is
$
\mathbb{E} \left[r_k^t|\bm{\gamma}^t\right] \geq \frac{1}{c_2\cdot NR}   r_k^t.
$
For all $k \in [NR]$,  it has
\begin{align} 
\mathbb{E} \left[r_k^t|\bm{\gamma}^t\right] \geq\frac{1}{c_2\cdot NR} \left( \sum_{\ell=1}^{NR} \frac{(r_\ell^t)^2  }{ NR}\right) \geq
 \frac{\left(\sum_{\ell=1}^{NR} r_{\ell}^t\right)^2}{(NR)^2c_2} \label{eq:decrease5}.
\end{align}
We plug the inequality $\epsilon(\bm{\gamma}^t) < \sum_{\ell}r_{\ell}^t$ in \eqref{eq:decrease5}, and get
\begin{align} 
\mathbb{E} \left[r_k^t|\bm{\gamma}^t\right]  \geq
 \frac{\epsilon^2(\bm{\gamma}^t)}{(NR)^2c_2}
. \label{eq:second_}
\end{align}

Based on \eqref{eq:second_} and  Jensen's inequality to check the induction hypothesis
\begin{align}\label{eq:decrease4_}
\mathbb{E} [\epsilon(\bm{\gamma}^{t+1})] - \mathbb{E}[\epsilon(\bm{\gamma}^{t})]
\leq ~ \mathbb{E}\left[r_k^t|\bm{\gamma}^t\right] \leq-\frac{\mathbb{E}[\epsilon(\bm{\gamma}^t)]^2}{(NR)^2c_2}.
\end{align}
The following proof for the convergence analysis for Algorithm \ref{tab:ML_coord} is similar to the proof for Algorithm  \ref{alg:Bandit} as discussed above. Hence, we omit the details here.

\section{Proof of Theorem \ref{th:N-arms}}\label{sec:proof_bandit}
In this section, we prove Theorem \ref{th:N-arms} which demonstrates the expected regret for the $N$-armed bandit problem in Algorithm \ref{alg:Thompson}. Recall that all arms are assumed to have Bernoulli distributed rewards, and that the first arm is the unique optimal arm. 

\textbf{Main technical arguments.}
Thompson sampling performs exploration
by selecting the arm with the best sampled mean to play. Therein, sampled means are generated from beta
distributions around the empirical means. As the number of plays of an arm increases, the beta distribution converges to the corresponding empirical mean. The main technical issue needed to be addressed in the analysis is that if the number of previous plays of the first arm is small, then the probability of playing the second arm will be as large as a constant even if it has already been played a large number of times. To address this, we introduce two types of arms, i.e., saturated and unsaturated arms, and bound the regret caused by each arm separately.
 Different from the previous analysis of Thompson sampling where the parameters of the beta distribution are required to be integral, i.e., \cite{agrawal2012analysis,scott2010modern}, our analysis applies to the beta distribution of which the parameters are in more general and natural form, represented in (\ref{eq:alpha}) and (\ref{eq:beta}).

\textbf{Notaions.} We take the inner 2-armed bandit problem in Algorithm \ref{alg:Thompson} as an example to illustrate corresponding notations in our paper. We denote $j_0$ as the number of plays of the first arm until \add{$T_p$} plays of the second arm. Denote $t_j$ as the time step where the $j$-th play of the first arm occurs (note that $t_0=0$). Furthermore, $Y_j\add{=t_{j+1}-t_j-1}$ is defined to characterize the number of time steps between the $j$-th and $(j+1)$-th plays of the first arm. The random variable $s_j$ is represented the number of successes in the first $j$ plays of the first arm. 

The random variable
$X(j,s,y)$ is defined to characterize the expectation of $Y_j$. To begin with, considering perform an experiment until it succeeds: examine if a 
\[\mbox{Beta}(s+R^j, j-s+R^j)\] distributed random variable surpasses a threshold $y$. Here, $R^j =r_{k_{t_j}}^{t_j}/F(\bm{\gamma}^{t_j})$ with $F(\bm{\gamma}^t)$ defined in (\ref{eq:mle_r}) and $r_{k_{t_j}}^{t_j}$ defined in (\ref{eq:reward}) is the reward obtained in Algorithm \ref{alg:Thompson} when the first arm of the inner MAB is played. For each experiment, the beta-distributed random variables are generated independently of the previous ones.  
We define $X(j,s,y)$ as the number of trials {before} 
the experiment succeeds. Thus, $X(j,s,y)$ is a random variable with parameter (success probability) $1-F^{beta}_{s+R^j,j-s+R^j}(y)$. Here $F^{beta}_{\alpha, \beta}$ denotes the cumulative distribution function (cdf) of the beta distribution with parameters $\alpha, \beta$. Also, let $F^B_{n,p}$ denote the cdf of the \emph{binomial} distribution with parameters $(n,p)$.

\textbf{Proof.} At any step $t$, we divide the set of suboptimal arms into two subsets: \emph{saturated} and \emph{unsaturated}. The saturated arm $i$ is the arm which have been played an enough large number ($L_i = c_L(\ln T)/\Delta^2_i$) for some large constant $c_L>0$ of times. The set of saturated arms at time $t$ is denoted as $C(t)$. Note that, for the set $C(t)$, with high probability, $\nu_i(t)$ is concentrated around $\mu_i$.
To bound the regret, we begin with estimating the number of steps between two consecutive plays of the first arm. After the $j$-th play of the first arm, the $(j+1)$-th play of the first arm will happen at the earliest time $t$ where $\nu_1(t) > \nu_i(t), \forall i\ne 1$. The number of steps before $\nu_1(t)$ is larger than $\nu_i(t)$ of each saturated arm $a \in C(t)$, and can be tightly approximated via a geometric random variable with the parameter being $\Pr(\nu_1 \ge \max_{a \in C(t)} \mu_i)$. We justify that the expected number of steps until the $(j+1)$-th play can be upper bounded by the product of the expected value of a geometric random variable $X(j,s_j, \max_i \mu_i)$, if $j$ plays of the first arm with $s_j$ have succeeded.
Additionally, the expected number of interruptions by the unsaturated arms is bounded by $\sum_{u=2}^N L_u$, since an arm $u$ becomes saturated after $L_u$ plays.

Based on the above discussion, the expected regret of the inner $q$-armed stochastic bandit problem in Algorithm \ref{alg:Thompson} can be bounded by the regrets due to unsaturated arms at saturated arms, given by
\begin{align}\label{eq:r}
\mathbb{E}[{\cal R}(T)]\leq
\mathbb{E}[{\cal R}_{\mathrm{uns}}(T)]+\mathbb{E}[{\cal R}_{\mathrm{s}}(T)].
\end{align}
Since an unsaturated arm $u$ becomes saturated after $L_u$ plays, the regret generated by unsaturated arms is bounded by
 \begin{align}
\mathbb{E}[{\cal R}_{\mathrm{uns}}(T)]\leq\sum_{u=2}^N L_u \Delta_u  = c_L (\ln T) \left(\sum_{u=2}^N \frac{1}{\Delta_u}\right),
 \end{align}
for some large constant $c_L>0$ .
Prior to bounding $\mathbb{E}[{\cal R}_{\mathrm{s}}(T)]$ in (\ref{eq:r}), we introduce some notations. Denote $\theta_j$ as the total number of plays of unsaturated arms in the interval between (and excluding) the $j^{th}$ and $(j+1)^{th}$ plays of the 
first arm. Thus the regret due to th play of the saturated arm can be approximated bounded by \cite{agrawal2012analysis}
\small\begin{align}\mathbb{E}[{\cal R}_{\mathrm{s}}(T)]
\leq C\cdot\left(\sum_{i=2}^q L_i\right)\cdot\left[\left(\sum_{j=0}^{\sum_{i} L_{i}} \sum_{i} \Delta_{i} \mathbb{E}\left[\min \left\{X\left(j, s_j, y_{i}\right), T\right\} | s_j\right]\right)\right],
\end{align}
for some constant $C>0$.

To complete the proof, the term $\mathbb{E}\left[\min \left\{X\left(j, s_j, y_{i}\right), T\right\} | s_j\right] =\frac{1}{1-F^{beta}_{s+R^j,j-s+R^j}(y)}-1 $ is required to be bounded.
Our proof is inspired by the paper \cite{agrawal2012analysis}. However, different from the previous analysis of Thompson sampling where the parameters of the beta distribution are required to be integral, i.e. \cite{agrawal2012analysis,scott2010modern}, our analysis applies to the beta distribution of which the parameters the beta distribution are in more general and natural form, represented in (\ref{eq:alpha}) and (\ref{eq:beta}). Hence, it yields that 
$\mathbb{E}[{\cal R}_{\mathrm{s}}(T)]\leq C\left(\left(\sum_{i} L_{i}\right)^{2}\right)=C\cdot\left(\sum_{i} \frac{\log T}{d_{i}^{2}}\right)^{2}$.
Hence, we conclude that 
\begin{align}
\mathbb{E}[{\cal R}(T)]\leq&
\mathbb{E}[{\cal R}_{\mathrm{uns}}(T)]+\mathbb{E}[{\cal R}_{\mathrm{s}}(T)]
\leq c_L (\ln T) \left(\sum_{u=2}^N \frac{1}{\Delta_u}\right)+C\cdot\left(\sum_{i} \frac{\log T}{d_{i}^{2}}\right)^{2}
=O\left(\left(\sum_{b=2}^q\frac{1}{d_b^2}\right)^2 \ln T\right).
\end{align}

	\bibliographystyle{ieeetr}

\begin{thebibliography}{}

\end{thebibliography}


\begin{thebibliography}{10}
	
	\bibitem{zanella2014internet}
	A.~Zanella, N.~Bui, A.~Castellani, L.~Vangelista, and M.~Zorzi, ``Internet of
	{T}hings for smart cities,'' {\em {IEEE} Internet Things J.}, vol.~1,
	pp.~22--32, Feb. 2014.
	
	\bibitem{8808168}
	K.~B. {Letaief}, W.~{Chen}, Y.~{Shi}, J.~{Zhang}, and Y.~A. {Zhang}, ``The
	roadmap to 6{G}: {AI} empowered wireless networks,'' {\em IEEE Communications
		Magazine}, vol.~57, pp.~84--90, Aug. 2019.
	
	\bibitem{8454392}
	L.~{Liu}, E.~G. {Larsson}, W.~{Yu}, P.~{Popovski}, C.~{Stefanovic}, and E.~{de
		Carvalho}, ``Sparse signal processing for grant-free massive connectivity: A
	future paradigm for random access protocols in the {I}nternet of {T}hings,''
	{\em IEEE Signal Process. Mag.}, vol.~35, pp.~88--99, Sep. 2018.
	
	\bibitem{hasan2013random}
	M.~Hasan, E.~Hossain, and D.~Niyato, ``Random access for machine-to-machine
	communication in {LTE}-advanced networks: Issues and approaches,'' {\em IEEE
		Commun. Mag.}, vol.~51, pp.~86--93, Jun. 2013.
	
	\bibitem{liu2018sparse}
	L.~Liu, E.~G. Larsson, W.~Yu, P.~Popovski, C.~Stefanovic, and E.~De~Carvalho,
	``Sparse signal processing for grant-free massive connectivity: A future
	paradigm for random access protocols in the {I}nternet of {T}hings,'' {\em
		IEEE Signal Process. Mag.}, vol.~35, pp.~88--99, Sep. 2018.
	
	\bibitem{arunabha2010fundamentals}
	A.~Ghosh, J.~Zhang, J.~G. Andrews, and R.~Muhamed, ``Fundamentals of {LTE},''
	{\em Englewood Cliffs, NJ, USA: Prentice-Hall}, 2010.
	
	\bibitem{chen2018sparse}
	Z.~Chen, F.~Sohrabi, and W.~Yu, ``Sparse activity detection for massive
	connectivity,'' {\em IEEE Trans. Signal Process.}, vol.~66, pp.~1890--1904,
	Jan. 2018.
	
	\bibitem{senel2018grant}
	K.~Senel and E.~G. Larsson, ``Grant-free massive {MTC}-enabled massive {MIMO}:
	A compressive sensing approach,'' {\em {IEEE} Trans. Commun.}, vol.~66,
	pp.~6164--6175, Aug. 2018.
	
	\bibitem{8761672}
	Z.~{Chen}, F.~{Sohrabi}, Y.~{Liu}, and W.~{Yu}, ``Covariance based joint
	activity and data detection for massive random access with massive {MIMO},''
	in {\em IEEE Int. Conf. Commun. (ICC)}, pp.~1--6, May 2019.
	
	\bibitem{liu2018massive1}
	L.~Liu and W.~Yu, ``Massive connectivity with massive {MIMO} part {I}: Device
	activity detection and channel estimation,'' {\em IEEE Trans. on Signal
		Process.}, vol.~66, pp.~2933--2946, Mar. 2018.
	
	\bibitem{haghighatshoar2018improved}
	S.~Haghighatshoar, P.~Jung, and G.~Caire, ``Improved scaling law for activity
	detection in massive {MIMO} systems,'' in {\em Proc. IEEE Int. Symp. Inf.
		Theory (ISIT)}, pp.~381--385, IEEE, 2018.
	
	\bibitem{chen2019phase}
	Z.~Chen and W.~Yu, ``Phase transition analysis for covariance based massive
	random access with massive {MIMO},'' in {\em Asilomar Conf. Signals Syst.
		Comput.}, 2019.
	
	\bibitem{wright2015coordinate}
	S.~J. Wright, ``Coordinate descent algorithms,'' {\em Math. Program.},
	vol.~151, no.~1, pp.~3--34, 2015.
	
	\bibitem{shalev2013accelerated}
	S.~Shalev-Shwartz and T.~Zhang, ``Accelerated mini-batch stochastic dual
	coordinate ascent,'' in {\em Proc. Neural Inf. Process. Syst. (NeurIPS)},
	pp.~378--385, 2013.
	
	\bibitem{shalev2013stochastic}
	S.~Shalev-Shwartz and T.~Zhang, ``Stochastic dual coordinate ascent methods for
	regularized loss minimization,'' {\em J. Mach. Learn. Res.}, vol.~14,
	no.~Feb, pp.~567--599, 2013.
	
	\bibitem{nutini2015coordinate}
	J.~Nutini, M.~Schmidt, I.~Laradji, M.~Friedlander, and H.~Koepke, ``Coordinate
	descent converges faster with the {G}auss-{S}outhwell rule than random
	selection,'' in {\em Proc. Int. Conf. Mach. Learn. (ICML)}, pp.~1632--1641,
	2015.
	
	\bibitem{salehi2018coordinate}
	F.~Salehi, P.~Thiran, and E.~Celis, ``Coordinate descent with bandit
	sampling,'' in {\em Proc. Neural Inf. Process. Syst. (NeurIPS)},
	pp.~9247--9257, 2018.
	
	\bibitem{perekrestenko2017faster}
	D.~Perekrestenko, V.~Cevher, and M.~Jaggi, ``Faster coordinate descent via
	adaptive importance sampling,'' in {\em Proc. Int. Conf. Articial
		Intelligence and Statistics (AISTATS)}, pp.~869--877, 2017.
	
	\bibitem{zhao2015stochastic}
	P.~Zhao and T.~Zhang, ``Stochastic optimization with importance sampling for
	regularized loss minimization,'' in {\em Proc. Int. Conf. Mach. Learn.
		(ICML)}, pp.~1--9, 2015.
	
	\bibitem{agrawal2012analysis}
	S.~Agrawal and N.~Goyal, ``Analysis of {T}hompson sampling for the multi-armed
	bandit problem,'' in {\em Proc. Conference On Learning Theory (COLT)},
	vol.~23, pp.~39.1--39.16, 2012.
	
	\bibitem{scott2010modern}
	S.~L. Scott, ``A modern {B}ayesian look at the multi-armed bandit,'' {\em Appl.
		Stoch. Model Bus.}, vol.~26, no.~6, pp.~639--658, 2010.
	
	\bibitem{wunder2015sparse}
	G.~Wunder, H.~Boche, T.~Strohmer, and P.~Jung, ``Sparse signal processing
	concepts for efficient 5{G} system design,'' {\em IEEE Access}, vol.~3,
	pp.~195--208, Feb. 2015.
	
	\bibitem{Yuanming_IoT2018device}
	T.~{Jiang}, Y.~{Shi}, J.~{Zhang}, and K.~B. {Letaief}, ``Joint activity
	detection and channel estimation for {I}o{T} networks: Phase transition and
	computation-estimation tradeoff,'' {\em IEEE Internet Things J.}, vol.~6,
	pp.~6212--6225, Aug. 2019.
	
	\bibitem{liu2018massive2}
	L.~Liu and W.~Yu, ``Massive connectivity with massive {MIMO} part {II}:
	Achievable rate characterization,'' {\em IEEE Trans. on Signal Process.},
	vol.~66, pp.~2947--2959, Mar. 2018.
	
	\bibitem{bubeck2012regret}
	S.~Bubeck, N.~Cesa-Bianchi, {\em et~al.}, ``Regret analysis of stochastic and
	nonstochastic multi-armed bandit problems,'' {\em Found. Trends Mach.
		Learn.}, vol.~5, pp.~1--122, Dec. 2012.
	
	\bibitem{chapelle2011empirical}
	O.~Chapelle and L.~Li, ``An empirical evaluation of {T}hompson sampling,'' in
	{\em Proc. Neural Inf. Process. Syst. (NeurIPS)}, pp.~2249--2257, 2011.
	
	\bibitem{mo2017channel}
	J.~Mo, P.~Schniter, and R.~W. Heath, ``Channel estimation in broadband
	millimeter wave {MIMO} systems with few-bit {ADC}s,'' {\em IEEE Trans. on
		Signal Process.}, vol.~66, pp.~1141--1154, Dec. 2017.
	
	\bibitem{7894211}
	S.~{Jacobsson}, G.~{Durisi}, M.~{Coldrey}, U.~{Gustavsson}, and C.~{Studer},
	``Throughput analysis of massive {MIMO} uplink with low-resolution {ADC}s,''
	{\em ‎IEEE Wireless Commun.}, vol.~16, pp.~4038--4051, Jun. 2017.
	
	\bibitem{wen2015bayes}
	C.-K. Wen, C.-J. Wang, S.~Jin, K.-K. Wong, and P.~Ting, ``Bayes-optimal joint
	channel-and-data estimation for massive {MIMO} with low-precision {ADC}s,''
	{\em IEEE Trans. Signal Process.}, vol.~64, pp.~2541--2556, Dec. 2015.
	
	\bibitem{bengio2018machine}
	Y.~Bengio, A.~Lodi, and A.~Prouvost, ``Machine learning for combinatorial
	optimization: a methodological tour d'horizon,'' {\em arXiv preprint
		arXiv:1811.06128}, 2018.
	
	\bibitem{8879693}
	Y.~{Shen}, Y.~{Shi}, J.~{Zhang}, and K.~B. {Letaief}, ``{LORM}: Learning to
	optimize for resource management in wireless networks with few training
	samples,'' {\em IEEE Trans. Wireless Commun.}, vol.~19, pp.~665--679, Jan.
	2020.
	
\end{thebibliography}

\end{document}